\documentclass[letterpaper, 10 pt, conference]{ieeeconf}  % Comment this line out if you need a4paper

\IEEEoverridecommandlockouts                              % This command is only needed if 
\usepackage{caption}
\usepackage{mathtools}
\usepackage{nomencl}
\usepackage{cite}
\usepackage{ntheorem}   % for theorems
\usepackage{lipsum}     % for sample text
\usepackage{graphicx}
\usepackage{algpseudocode}

\usepackage{algorithm}

\usepackage{esint}
\usepackage[latin9]{luainputenc}
\usepackage{amssymb}
\usepackage{lineno,hyperref}
\usepackage{xcolor}

\modulolinenumbers[5]
\usepackage{bm}
\usepackage{amsmath}

\DeclareRobustCommand{\uvec}[1]{{%
		\ifcsname uvec#1\endcsname
		\csname uvec#1\endcsname
		\else
		\bm{\mathbf{#1}}%
		\fi
}}

\newtheorem{prop}{Proposition}
\newtheorem{rem}{Remark}

\title{\LARGE \bf
A Novel Assistive Controller Based on Differential Geometry for Users of the Differential-Drive Wheeled Mobile Robots 
}

\author{Seyed Amir Tafrishi, Ankit A. Ravankar, Jose Victorio Salazar Luces
	 and Yasuhisa Hirata % <-this % Mstops a space
	\thanks{Authors are with Department of Robotics, Tohoku University, 6-
		6-01 Aramaki-Aoba, Aoba-ku, Sendai 980-8579, Japan. 
		{\tt\small \{s.a.tafrishi, ankit, j.salazar, hirata\}@srd.mech.tohoku.ac.jp}	}%
}
 \newcommand{\makehighlight}[1]{\textcolor{black}{#1}}
\begin{document}

\maketitle
\thispagestyle{plain}
\pagestyle{plain} %change to empty for main manuscript

%Differentially Wheeled 
%%%%%%%%%%%%%%%%%%%%%%%%%%%%%%%%%%%%%%%%%%%%%%%%%%%%%%%%%%%%%%%%%%%%%%%%%%%%%%%%
\begin{abstract}
\makehighlight{Certain wheeled mobile robots e.g., electric wheelchairs, can operate through indirect joystick controls from users.} Correct steering angle becomes essential when the user should determine the vehicle direction and velocity, in particular for differential wheeled vehicles since the vehicle velocity and direction are controlled with only two actuating wheels. This problem gets more challenging when complex curves should be realized by the user. A novel assistive controller with safety constraints is needed to address these problems. Also, the classic control methods mostly require the desired states beforehand which completely contradicts human's spontaneous decisions on the desired location to go. \makehighlight{In this work, we develop a novel assistive control strategy based on differential geometry relying on only joystick inputs and vehicle states where the controller does not require any desired states.} We begin with explaining the vehicle kinematics and our designed Darboux frame kinematics on a contact point of a virtual wheel and plane. Next, the geometric controller using the Darboux frame kinematics is designed for having smooth trajectories under certain safety constraints. We experiment our approach with different participants and evaluate its performance in various routes.
\end{abstract}

%%%%%%%%%%%%%%%%%%%%%%%%%%%%%%%%%%%%%%%%%%%%%%%%%%%%%%%%%%%%%%%%%%%%%%%%%%%%%%%%
\section{Introduction}
Recently, different vehicles are increasingly getting involved in human life and activities. The wheeled mobile robot e.g., wheelchairs are the conventional vehicles utilized more often. \makehighlight{ There have been researches to control these mobile robots considering their desired states.} However, the vehicle users are not always fully aware of the exact determined goals (desired configuration/states) or they change their goals continuously. \makehighlight{This urges a controller that assists the user input without any external sensors or a priori information about the desired states}. For example, the assistive controller can be greatly beneficial in wheelchair patients who do not have full ability to control certain parameters of the moving mobile robot (varying the vehicle velocity). These issues are barely covered in the literature.  

%% The Trajectory tracki and path pllaning 
\makehighlight{We can roughly divide the motion control problem for the mobile robot into path planning and trajectory tracking \cite{261508,badreddin1993fuzzy,laumond1994motion,Astolfi1999} with some about sensor-based planning algorithms \cite{gasparetto2015path}}. In the path planning problem, the vehicle's initial and desired configuration consisting of the position and direction of the vehicle \makehighlight{are often known} \cite{zheng1993recent,laumond1998robot}. The goal is to establish a suitable trajectory for the vehicle under certain constraints namely known obstacles. For instance, a nilpotent form of the model was designed to create feedback transformation for different systems including car systems \cite{261508}. Laumond et al. developed various geometric control strategies for collision-free path planning of wheeled mobile robots \cite{laumond1994motion,laumond1998robot,giordano2006nonholonomic}. Certain planning techniques were proposed for more challenging cases \cite{hsu2002randomized,lavalle2006planning,minguez2016motion}, namely, path planning through moving obstacles. %% Navigation through different method motion field  
\makehighlight{Recent studies took a direction in making algorithmic motion planning by combining different sensors e.g., LiDAR and depth camera through navigating crowded environments \cite{gasparetto2015path,paden2016survey}. The ultimate goal is to achieve autonomous self-driving cars \cite{paden2016survey,badue2021self} that are covered in recent studies. }
 %% The path planning studies. 
 
  \begin{figure} 
 	\centering
 	\vspace{3mm} %5mm vertical space
 	\includegraphics[width=3.2 in, height=2 in]{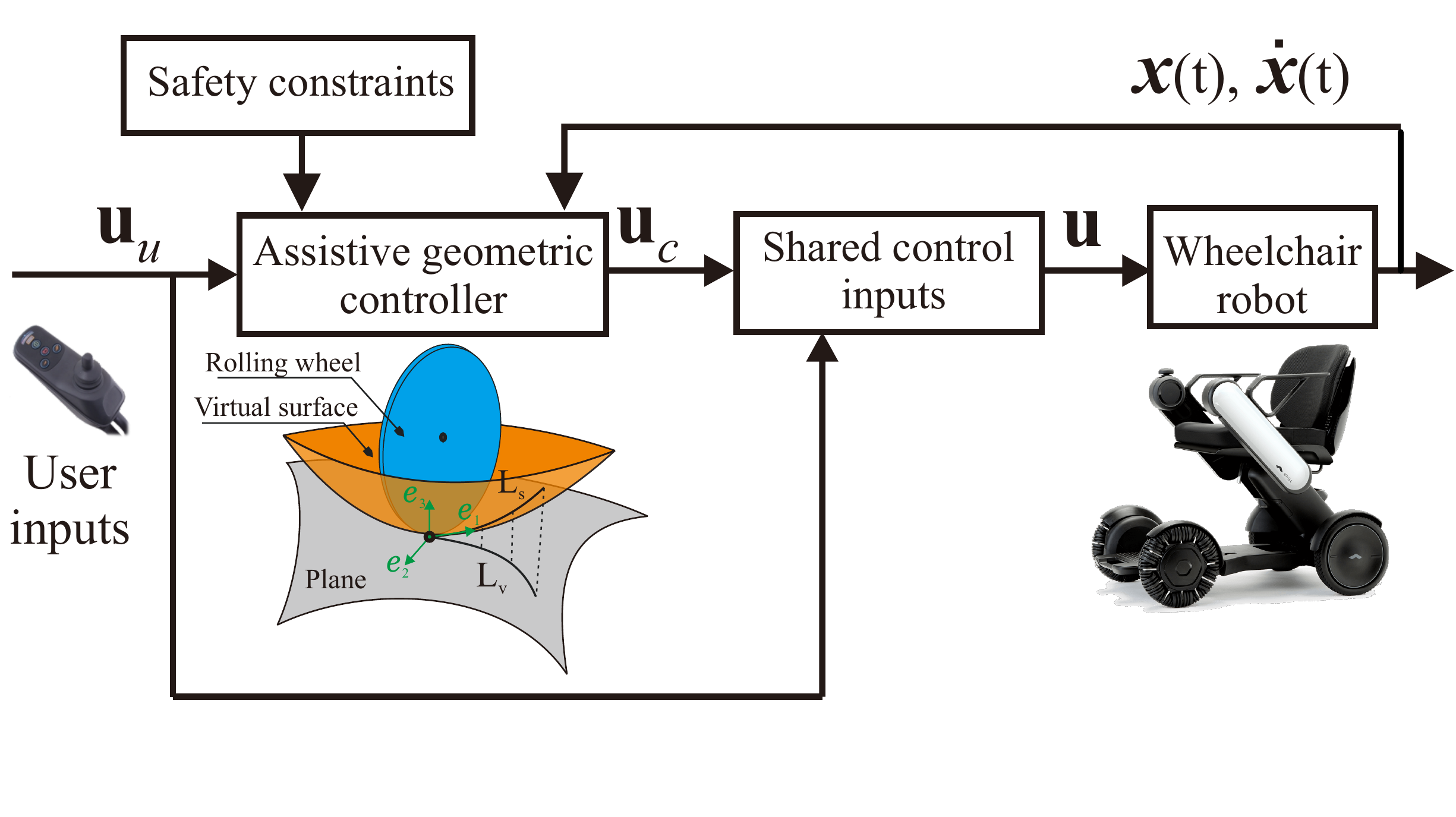}
 	\caption{The framework of the proposed assistive controller. }
 	\label{Fig:blockdiagram}
 \end{figure}
\makehighlight{The trajectory tracking problem aims for tracking the already given path or desired states, velocity in particular, on the current time \cite{samson1995control,laumond1998robot,morin2008motion,tzafestas2018mobile}. In other words, the objective is to keep the vehicle on the desired trajectory or reference states.} This control problem focuses on minimizing the external disturbances and it often uses robot dynamics to minimize the convergence errors. For example, a combined feedforward and feedback control strategy was created for a mobile robot under disturbances \cite{klancar2005mobile}. Currently, there have been certain attempts to do reference tracking by linear control methods \cite{sousa2016trajectory,forte2018reference} or fuzzy logic control \cite{abdelwahab2020trajectory}. \makehighlight{In the control perspective, there have been attempts to develop strategies that make the user travel with the vehicle safely by combining different sensors.  The reactive control strategy with the inclusion of LiDAR sensor data is one of the conventional ways \cite{prassler2001robotics,gasparetto2015path}.
%% Trajectory tracking (reactive methods) sensor based 
Based on the covered studies, there is no controller that does not require a priori reference trajectory or desired configuration without using many sensor e.g., LiDAR information. This means, an assistive controller can be highly important to correct the user inputs for safer and better locomotion of the vehicle by using only joystick inputs. Also, this control strategy cannot easily be developed with scaling the inputs since the trajectories can get complex with following continuously changing user inputs and the problem becomes harder.}  

\makehighlight{There have been some attempts in developing shared autonomy between user inputs and manipulator robots \cite{dragan2013policy,javdani2015shared,javdani2018shared} outside of the mobile robot research field.}

S. Javdani et al. used a partially observable Markov decision process (POMDP) to predict the intended object to pick while the autonomy system does not know the goal (without desired configuration) as a priori \cite{javdani2018shared}. The work shed light on control strategies with no desired states but still had challenges since the manipulator mainly ran with an internal controller and did not consider any complex trajectory utilization that the user intended to do with a joystick. This important issue is predominant in mobile robots because user can follow complex and highly dynamic paths with spontaneous decisions on the desired goal.

  %% The different contollers trajectory tracking

\makehighlight{Various control strategies have been considered for wheelchair systems. Park and Kuipers developed a feedback controller in order to drive an electric-powered differential-drive wheelchair \cite{park2011smooth}. This controller was correcting the vehicle velocity depending on the path's curvilinear curvatures. There have been studies in smoothing the traversed path by mobile robots \cite{ravankar2018path}. Also, a velocity control to follow a reference trajectory was designed considering jerk limitation and longitudinal acceleration \cite{seki2005velocity}. However, complex curves that have sharp turns/features were not considered in these studies for increasing the level of maneuvers. Next, a predictive controller with a linearized model was proposed for trajectory tracking of an electric wheelchair \cite{nguyen2018path}. Previously proposed control approaches for a wheeled mobile robot are not well suited for wheelchair users since the user/patient cannot determine and communicate with an advanced controller to feed desired states (this can be velocity or position) continuously. This problem multiplies with using discrete joystick controls (easy to jump different values continuously) in wheelchairs rather than an advanced mechanical shaft to the steering wheel that cars normally encompass.}
% There were different researches in improving the wheelchair mechanism \cite{gonzalez2009kinematic,chen2011wheelchair,leaman2017comprehensive} and driving unit \cite{Takayuki2006,hirata2013regenerative,hirata2017motion} that the basic model of the wheelchair was fixed using conventional control strategies.

In our previous study, we developed a simple geometric controller for a mobile robot with steering wheel and evaluated the motion in simulation \cite{TafrishiRObio2021}. In this work, we introduce a new assistive controller for differential-drive mobile robot [see Fig. \ref{Fig:blockdiagram} for the framework], with motivation from existing challenges, that does not use any priori desired configuration/states. The established new geometric controller $\uvec{u}_c$ relies on the curvature variation of the traversing trajectory using the Darboux frame kinematics of a virtual wheel \cite{CuiDarboux2010,TafrishiMAMT2021}. \makehighlight{Also, we design this controller with arc-length-based time-invariant functions that rely only on user joystick inputs $\uvec{u}_u$ and feed on current states of the vehicle $\{\uvec{x}(t),\dot{\uvec{x}}(t)\}$ with using inertial measurement unit (IMU) and shaft encoders, under defined safety constraints.}

 In this paper, we begin by describing the vehicle and Darboux frame kinematics in Section II. Next, based on our defined control problem in Section III, we derive our geometric controller on the Darboux frame kinematics in Section IV. Finally, we experiment and check the controller behavior and performance with different participants by using a differential-drive wheelchair in Section V. \makehighlight{To the best of our knowledge, this is the first assistive controller approach that does not require any priori desired states and relies on user joystick input and vehicle current states only without utilizing external sensors such as LiDAR or camera. The new controller tries to develop smooth trajectories while it is decreasing the user effort. Thus, it can be extensively suitable for mobile robot applications.}

%We have organized the paper as follows. In Section II, the kinematics of the differential wheeled mobile robot and the Darboux frame kinematics between a virtual wheel and plane for the purpose of developing the assistive controller are described. In Section III, the problem statement including the objectives of the new assistive controller is explained. We then develop our geometric controller utilizing the Darboux frame kinematics in Section IV. In Section V, the controller is experimented with a electric wheelchair and the performance of our approach is evaluated by different participants. Finally, we conclude our findings in Section VI. 
\section{Kinematics of the Vehicle and Virtual Wheel on a Darboux Frame}
In this section, we demonstrate the kinematic model for differential wheel mobile robots. Next, a Darboux frame with a virtual wheel is introduced at the vehicle center since it will be used for designing an assistive control with geometric functions in the arc-length domain.
 
  \begin{figure} 
	\centering
	\vspace{3mm} %5mm vertical space
	\includegraphics[scale=.54]{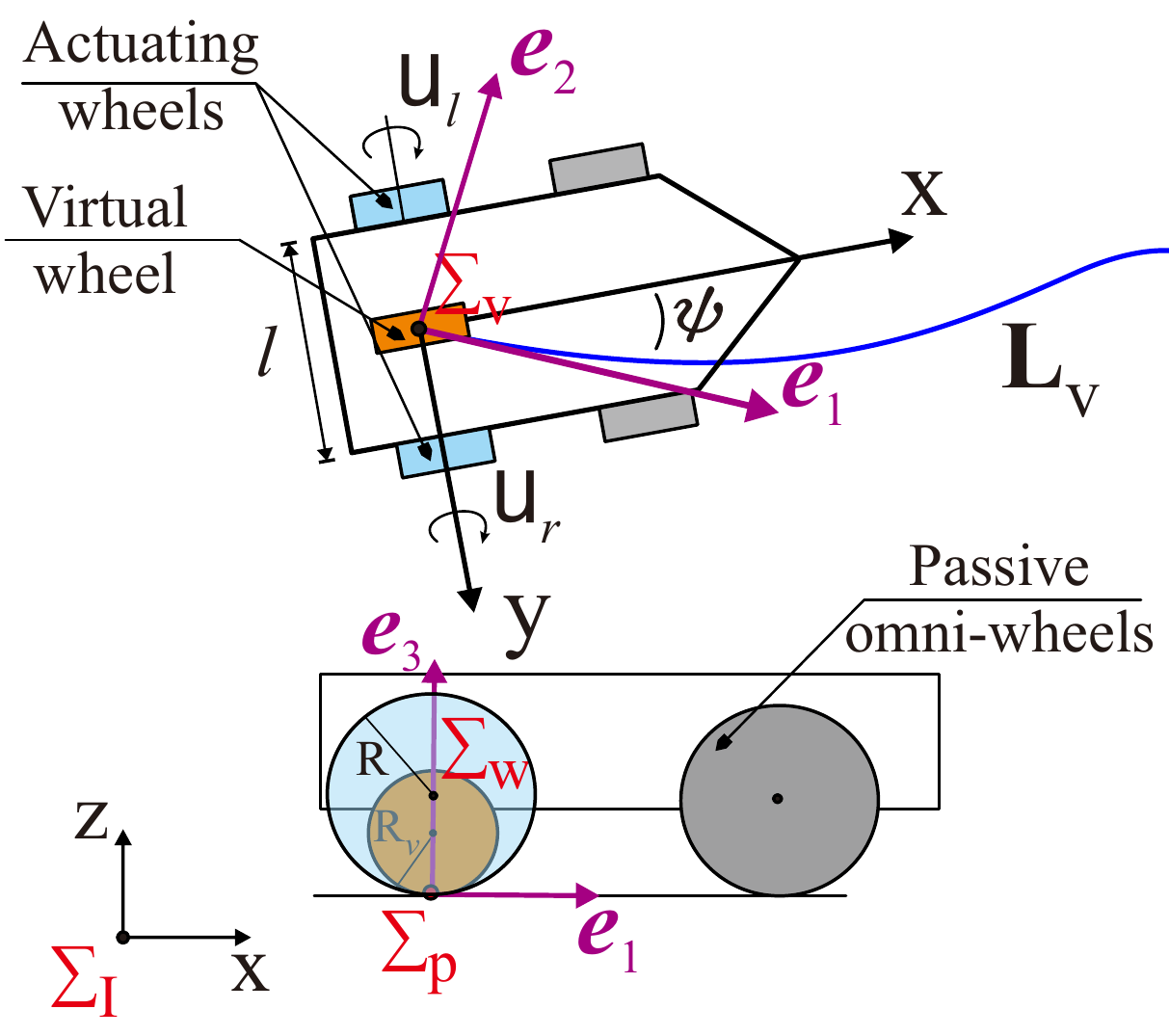}
	\caption{\makehighlight{Schematic of the differential wheeled mobile robot with virtual wheel frame at the center of the vehicle.}}
	\label{Fig:car_schematics}
\end{figure}
Fig. \ref{Fig:car_schematics} illustrates the frames of the differential wheeled mobile robot. We have simplified the model for generality to a moving vehicle where the robot is actuated with two wheels in the back and passive omnidirectional wheels at the front without steering. The frame $\Sigma_V$ is fixed to the vehicle center where it is aligned with the actuating back wheels center with a frame of $\Sigma_W$. Additionally, $\Sigma_I$ is the inertia frame. The vehicle axis $\Sigma_V$ has a spin angle of $\psi$ with respect to the inertia frame. In addition to classic frames, we introduce a new frame on the vehicle $\Sigma_V$ located at the contact point of wheels with the plane as a virtual wheel frame $\Sigma_p$. Next, we assume that $\uvec{L}_v$ \makehighlight{with $C^{\infty}$ smoothness in continuous-time domain} is a path that the wheeled mobile robot follows on the plane. Also, the back driving wheels with a radius of $R$ are rotating. Note that there is a no-sliding constraint between the wheel and the ground.

\makehighlight{Based on the frame definitions, the kinematics of the moving differential wheeled vehicle $\Sigma_V$ in time $t>0$} on the plane can be presented as follows
\begin{eqnarray}
\dot{\uvec{x}}(t)=\left[\begin{array}{c}
\dot{x}(t)\\
\dot{y}(t)\\
\dot{\psi}(t)
\end{array}\right]=\left[\begin{array}{c}
\cos\psi\\
\sin\psi\\
0
\end{array}\right]u_v+\left[\begin{array}{c}
0\\
0\\
1
\end{array}\right]u_\omega ,
	 \label{Eq:Kinematicofvehicle}
\end{eqnarray}
where $(x,y)$ and $\psi$ are the plane states and the rotation angle between the vehicle and the plane and also the control inputs $\uvec{u}=[u_v,u_{\omega}]^T$ in system (\ref{Eq:Kinematicofvehicle}) are the vehicle linear velocity $u_v$ and angular velocity $u_\omega$ inputs which is given as
\begin{eqnarray}
\left[\begin{array}{c}
u_v\\
u_\omega
\end{array}\right]=\left[\begin{array}{cc}
1/2& 1/2 \\
1/l & -1/l
\end{array}\right]\left[\begin{array}{c}
u_r\\
u_l
\end{array}\right],
\end{eqnarray}
where $u_r$ and $u_l$ are the angular velocity inputs of the left and right wheels.

Because we want to develop a geometric controller based on the traversing path's curvature, we use the Darboux frame definition \cite{CuiDarboux2010,DIfgeometry1976,TafrishiMAMT2021}. In our early study \cite{TafrishiMAMT2021}, we illustrated that how the Darboux frame can be used in parameterizing the motion kinematics with arc-length-based inputs for a spin-rolling sphere on the plane. \makehighlight{The simpler case of this system as the rolling disc (wheel in our case) was also shortly discussed \cite{CuiDarboux2010,TafrishiMAMT2021}. The Darboux-frame kinematics brings different advantages such as producing velocities based on the projected curvature of the interacting surfaces. This is not possible with Frenet frame \cite{samson1995control} kinematics which studies the moving particle on curves. In our case, the interacting surfaces are a rolling object (virtual wheel here), the plane surface, and sandwiched virtual surface on the Darboux frame. The virtual surface is then designed as arc-length-based inputs. In addition, the Darboux-frame-based kinematics is time- and coordinate-invariant and can be parametrized with extra control inputs that help to increase accessibility \cite{TafrishiMAMT2021} which we utilize in this study.}

Darboux frame kinematics is derived on $\Sigma_p$ frame. It is assumed that the Darboux frame $\Sigma_p$ is at the center of the vehicle where a virtual wheel with radius $R_v$ rotates at $\uvec{P}$. Every contacted point $\uvec{P}$ on $\Sigma_p$ has a unit-based Darboux frame ($\bm{e}_1,\bm{e}_2,\bm{e}_3$) \cite{Riemannian2002} which follows the path $\uvec{L}_v$ where $\bm{e}_1$ is a tangent vector to the path $\uvec{L}_{v}$, $\bm{e}_3$ is a normal vector to the plane, and $\bm{e}_2$ is perpendicular to the plane, $\bm{e}_3 \times \bm{e}_1$. Then, the general angular velocity of the wheel (rolling disc) $\bm{\omega}_p$ on the contact point of the Darboux frame is obtained \cite{CuiDarboux2010,TafrishiMAMT2021}
  \begin{align} 
  \bm{\omega}_p= \delta  (-\tau_g \bm{e}_1+k_n\bm{e}_2-k_g\bm{e}_3),
  \label{Eq:angularveloctytotal}
  \end{align}
where 
\begin{align}
&\delta=ds/dt,\;k_g= k^{w}_g-k^{s}_g-\alpha_s,\;k_n=k^{w}_n-k^{s}_n-\gamma_{s},\;\nonumber\\ &\tau_g=\tau_g^{w}-\tau^{s}_g,
\label{EQ:TheCurvaturerelevantDif}
\end{align}
where $k^w_g=0$, $\tau^w_g=0$ and $k^w_n=1/R_v$ are virtual wheel's geodesic curvature, geodesic torsion and normal curvature, and the fixed surface is considered as a plane with $k^s_g=\tau^s_g=k^s_n=0$ curvatures. Note that here radius of virtual wheel $R_v$ is equal to our considered wheels radius $R$. Also, $\{\alpha_s,\gamma_s\}$ are arc-length-based inputs as a sandwiched virtual surface \cite{TafrishiMAMT2021} and $\delta$ is defined as the rolling rate input in the time domain. \makehighlight{ We will utilize the extended control inputs $\{\alpha_s,\gamma_s,\delta\}$ in developing our controller. In a previous study, we determined how an underactuated ball-plate system is transformed to a fully-actuated model with extra arbitrary inputs \cite{TafrishiMAMT2021}. We take similar and simplified computation to derive the rolling virtual wheel to transform the system (\ref{Eq:Kinematicofvehicle}) to 3$\times$3. The extra inputs help us to define extra functions in constraining the traveled path and vehicle velocity dependent on our safety constraints (derived functions).} By substituting the curvatures into Eqs. (\ref{Eq:angularveloctytotal})-(\ref{EQ:TheCurvaturerelevantDif}), we have the angular $\bm{\omega}_p$ and linear $\uvec{v}_p$ velocities on $\uvec{P}$ as follows 
\begin{eqnarray}
\bm{\omega}_p&=& \delta \left[ \left[(1/R_v)+\gamma_s\right]\bm{e}_2 +\alpha_s \bm{e}_3 \right],\; \nonumber\\
 \uvec{v}_p&=& \bm{\omega}_p\times \uvec{r}_w= \delta \left[ \left[(1/R_v)+\gamma_s\right]\bm{e}_2 +\alpha_s \bm{e}_3 \right] \times R_v\;\bm{e}_3 \nonumber \\
&=&\delta \left[1+ R_v \; \gamma_s  \right]\bm{e}_1.
\label{Eq:Rollingdiscunitbased}
\end{eqnarray}
%where $\bm{\omega}_p$ can be considered equal to $\dot{\psi}\bm{k} $.
%Note that the wheel angular velocity $\bm{\omega}_w$ is equal to Darboux frame angular velocity $\bm{\omega}_p$ due to the no-sliding constraint. 
We can see that the curvature variation is projected onto unit frames which varies the linear and angular velocities. The virtual wheel travels on the plane along $\bm{e}_1$ with corresponding kinematics (\ref{Eq:Rollingdiscunitbased}). Additionally, the kinematics can be represented in the classic velocity formulation on the Cartesian coordinate $\{\uvec{p}-\bm{i}\bm{j}\bm{k}\}$ system by following expressions
\begin{align}
&\bm{e}_1=\cos \psi \bm{i}+\sin \psi \bm{j},\;\;\bm{e}_2=-\sin \psi \bm{i}+\cos\psi \bm{j},\; \bm{e}_3=\bm{k}, \nonumber \\
&\delta =\frac{ds}{dt}= R_v  \frac{d \theta}{dt}.
\label{Eq:examplecasecartesianveldar}
\end{align}
where $\theta$ represents angular rotation of the virtual wheel.
    \begin{figure} 
	\centering
	\vspace{3mm} %5mm vertical space
	\includegraphics[scale=.70]{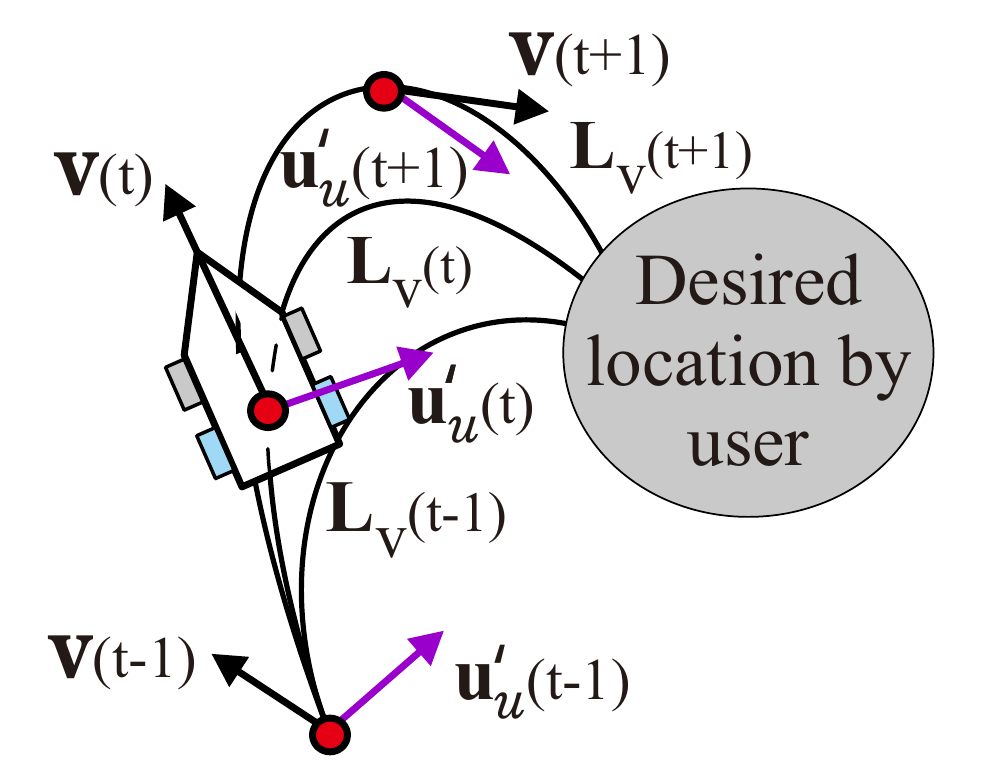}
	\caption{\makehighlight{The problem statement of the motion assisted control. The vehicle with the velocity $\uvec{v}(t)$ always creates safe and smooth trajectories in continues-time $t$ on $\uvec{L}_v$ while the user input $\uvec{u}'_u$ aims for the desired location.}}
	\label{Fig:Motion_Action_Trajecty}
\end{figure}
\section{Problem Statement}
\label{Sec:ProblemStatement}
We explain the control problem that we have established for the wheeled mobile robot. Also, the safety conditions as constraints are described.

 \makehighlight{Fig. \ref{Fig:Motion_Action_Trajecty} presents an example statement that we have chosen. The assistive controller aims to correct the moving vehicle velocities $\uvec{v}(t)=[v_v,\;\dot{\psi}]^T$ under safety constraints without knowing any priori goal where $v_v$ is the linear velocity of the vehicle. Also this controller allows the user to have more freedom in assigning travel velocity $u'_{v,u}$ and angular direction $u'_{\psi,u}$ by the joystick inputs.} Next, the inputs $\uvec{u}=[u_v,u_{\omega}]^T$ of the kinematic model (\ref{Eq:Kinematicofvehicle}) are defined with the following shared control formulation
\begin{equation}
%\label{Eq:theControllerMainvelocity}
\uvec{u}(t)=(1/n)\left[\uvec{u}_u(t)+(n-1)\uvec{u}_c\right],
\label{Eq:theControllerMainDefi}
\end{equation} 
where $\uvec{u}_u=[u_{v,u},u_{\omega,u}]^T$, $\uvec{u}_c=[u_{v,c},u_{\omega,c}]^T$ and $n \in N$ are the raw inputs by the user, the inputs of the controller, and the level of reliance on the controller. \makehighlight{The higher shared control variable $n\in [1,\infty]$ becomes the less the user inputs $\uvec{u}_u$ will be effective in directing the vehicle, and vehicle will be assisted by only the controller inputs $\uvec{u}_c$. }

In our designed controller, we do not have a desired state/configuration. As shown in Fig. \ref{Fig:blockdiagram}, only  the user's raw inputs $\uvec{u}'_u=(u'_{v,u},u'_{\psi,u})$ and vehicle current states $\{\uvec{x},\dot{\uvec{x}}\}$ are used for the controller functions. \makehighlight{Also, the vehicle states, velocity $\uvec{v}$ and current orientation of the vehicle $\psi(t)$ are obtained by measurements from inertial measurement unit and motor shaft encoders.} Please note that in the case of mobile robot without controller $\uvec{u}_u=\uvec{c}\uvec{u}'_u$, where $\uvec{c}$ is a linear ratio to transform user joystick inputs to velocity. However, in our case, our nonlinear geometric controller utilizes the raw inputs of user $\uvec{u}'_u$ to develop a proper strategy in updating $\uvec{u}_c$.

We assume, only the user knows the temporary desired configuration $\uvec{x}_f=\{x_f(t),y_f(t),\psi_f(t)\}$ while vehicle moves along arbitrary $\uvec{L}_v$ path (path is not known/given). The controller velocity input and steering angle $\uvec{u}_c$ are updated from an intended curvature of trajectory $\uvec{L}_v$ and introduced objectives based on the user/driver inputs $\uvec{u}'_u=[u'_{v,u},u'_{\psi,u}]^T$. Next, the following safety and control objectives are established:

\makehighlight{(i)} Only the user knows the approximate desired configuration $\uvec{x}_f$ and our shared controller $\uvec{u}$ is assisted by including the user inputs $\uvec{u}_u$ without any priori knowledge of $\uvec{x}_f$.

\makehighlight{(ii)} The controller corrects the trajectory for having smooth curvature on path $\uvec{L}_v$ while the motion safety is satisfied under geometric motion constraints on input $\uvec{u}$. The assistive controller is based on the defined geometric constraints through the virtual wheel model (\ref{Eq:Rollingdiscunitbased}) on the Darboux frame that uses the user inputs $\uvec{u}_u$.

\makehighlight{(iii)} The first safety constraint is on the input velocity $u_{v,u}$ that changes under \makehighlight{$|(1/n)[u_{v,u}+(n-1)u_{v,c}(\uvec{u}'_{u},\uvec{v})]| \leq | v_{v} |$} condition where $u_v(t)$ is the corrected mobile robot velocity that uses the steering angle of the user $u'_{\psi,u}$ with the respect to the changes of smooth path curvature of $\uvec{L}_v$.

\makehighlight{(iv)} The sensitivity of the vehicle orientation $\psi$ is considered as the second safety constraint. The input $u_{\omega,u}$ has a safety constraint that varies based on the vehicle velocity $\uvec{v}(t)$ and the difference between the current intended steering angle of the user steering input $u'_{\psi,u}$ and vehicle orientation $\psi(t)$ as \makehighlight{$|\int_{0}^{t} \left[ (1/n)[u_{\omega,u}+(n-1)u_{\omega,c}(\uvec{u}'_u,\uvec{v},\psi)]\right] dt|\leq|\psi|$}.

%The assistive controller aims to help the user while the person controls the differential-drive wheeled mobile robot towards its desired configuration with safer and better locomotion. The better locomotion is denoted with velocity variation based on the given steering angle on the smooth trajectories. These safety constraints have importance since improper steering in high velocities will cause the vehicle to topple or do dangerous spinning around itself. Therefore, we have proposed a new control problem that is different from conventional problem statements (controller dependence on the desired states) in the literature. 
\section{Assistive Geometric Controller}
In this section, we design our assistive geometric controller by utilizing differential geometry. The geometric functions of the controller in the arc-length domain are designed under the explained objectives in our problem statement.

We want to construct our geometric controller based on the introduced Darboux frame kinematics of the virtual wheel. \makehighlight{We want to develop a relationship between the system inputs (\ref{Eq:Kinematicofvehicle}) and our proposed Darboux-frame-based kinematics \cite{TafrishiMAMT2021}. The idea is first to transform an underactuated (3x2) system (\ref{Eq:Kinematicofvehicle}) to a fully-actuated one using the introduced Darboux-frame kinematic (\ref{Eq:Rollingdiscunitbased}) on $\Sigma_p$ where it is time- and coordinate-invariant. This helps the system to be more accessible as indicated in our previous work \cite{TafrishiMAMT2021} where a generic ball-plate system was considered. Then, The induced geometric surface as the virtual surface $\{\alpha_s,\gamma_s\}$ on $\Sigma_p$ can be designed as a controller for system (\ref{Eq:Kinematicofvehicle}) in an arc-length domain where rolling rate $\delta$ variable separates time domain from the controller.}
\begin{prop}
	\label{PropostionDarbouxControlinput}
Let's assume there is a no-sliding constraint. The Darboux-frame-based kinematics (\ref{Eq:angularveloctytotal}) with the relationship between rolling virtual wheel, the plane and virtual surface is considered in separated arc-length domain \cite{TafrishiMAMT2021}. \makehighlight{Then, the geometric controller as the sandwiched virtual surface on $\Sigma_p$ with aligned $\{\bm{e}_1,\bm{e}_3\}$ unit vectors with vehicle frame $\Sigma_V$ is defined as follows }
	\begin{equation}
	\label{Eq:Geoemtricocontrolvector}
	\uvec{u}_c= \left[\begin{array}{c}
	u_{v,c}\\
	u_{\omega,c}
	\end{array}\right] = \left[\begin{array}{c}
	\delta \left(1+ R_v \; \gamma_s  \right)\\
	\delta \alpha_s
	\end{array}\right],
	\end{equation}
	where $(\alpha_s,\gamma_s)$ are the virtual surface inputs and $\delta$ is the rolling rate input. 
	\end{prop}
\begin{proof}
The velocity input $u_v$ corresponds to the vehicle velocity that is always aligned to unit vector $\bm{e}_1$ of the Darboux frame. Since there is a no-sliding constraint between the wheel and ground, one can find one of the geometric controller inputs $u_{v,c}\bm{e}_1$ in the function of the arc-length-based input $\gamma_s$ with the Darboux frame linear velocity $\uvec{v}_p$ (\ref{Eq:Rollingdiscunitbased}) and (\ref{Eq:Kinematicofvehicle}) as follows
\begin{equation}
u_{v,c}(\delta,\gamma_s)=\delta \left[1+ R_v \; \gamma_s  \right].
\label{Eq:UWCgeometriccontroler}
\end{equation}

For the next control input, \makehighlight{Based on the properties of the Darboux frame $\Sigma_p$ definition that is located on the center of vehicle $\Sigma_V$, the angular velocity in differential equation from (\ref{Eq:Kinematicofvehicle}) is equal to the geodesic torsion input $\alpha_s $ of $\bm{\omega}_p$ in (\ref{Eq:Rollingdiscunitbased}) by
\begin{equation}
\delta \alpha_s \triangleq \dot{\psi}(t),%= \big (u_{\psi,c}/l \big ) \tan u_{v,c},
\label{Eq:definitionspindarbou}
\end{equation}
because the rotation of the vehicle is always around $\bm{e}_3$ vector.} Now, by re-ordering Eq. (\ref{Eq:definitionspindarbou}) and knowing $\dot{\psi}=u_{\omega,c}$ from (\ref{Eq:Kinematicofvehicle}) give
\begin{equation}
u_{\omega,c}(\delta,\alpha_s)=\delta \alpha_s.
\label{Eq:UPSICgeometricontroler}
\end{equation}
Eqs. (\ref{Eq:UWCgeometriccontroler}) and (\ref{Eq:UPSICgeometricontroler}) give us the geometric controller $\uvec{u}_c$ in Eq. (\ref{Eq:Geoemtricocontrolvector}) with the arc-length-based inputs ($\alpha_s,\gamma_s$) of the virtual surface and rolling rate input $\delta$. 
%Note that Eq. (\ref{Eq:Geoemtricocontrolvector}) together with the vehicle model (\ref{Eq:Kinematicofvehicle}) and (\ref{Eq:theControllerMainDefi}) present the complete system. 
\end{proof}
 
\begin{figure}[t!]
	\centering	
	\vspace{3mm} %5mm vertical space
	\includegraphics[width=2.4 in]{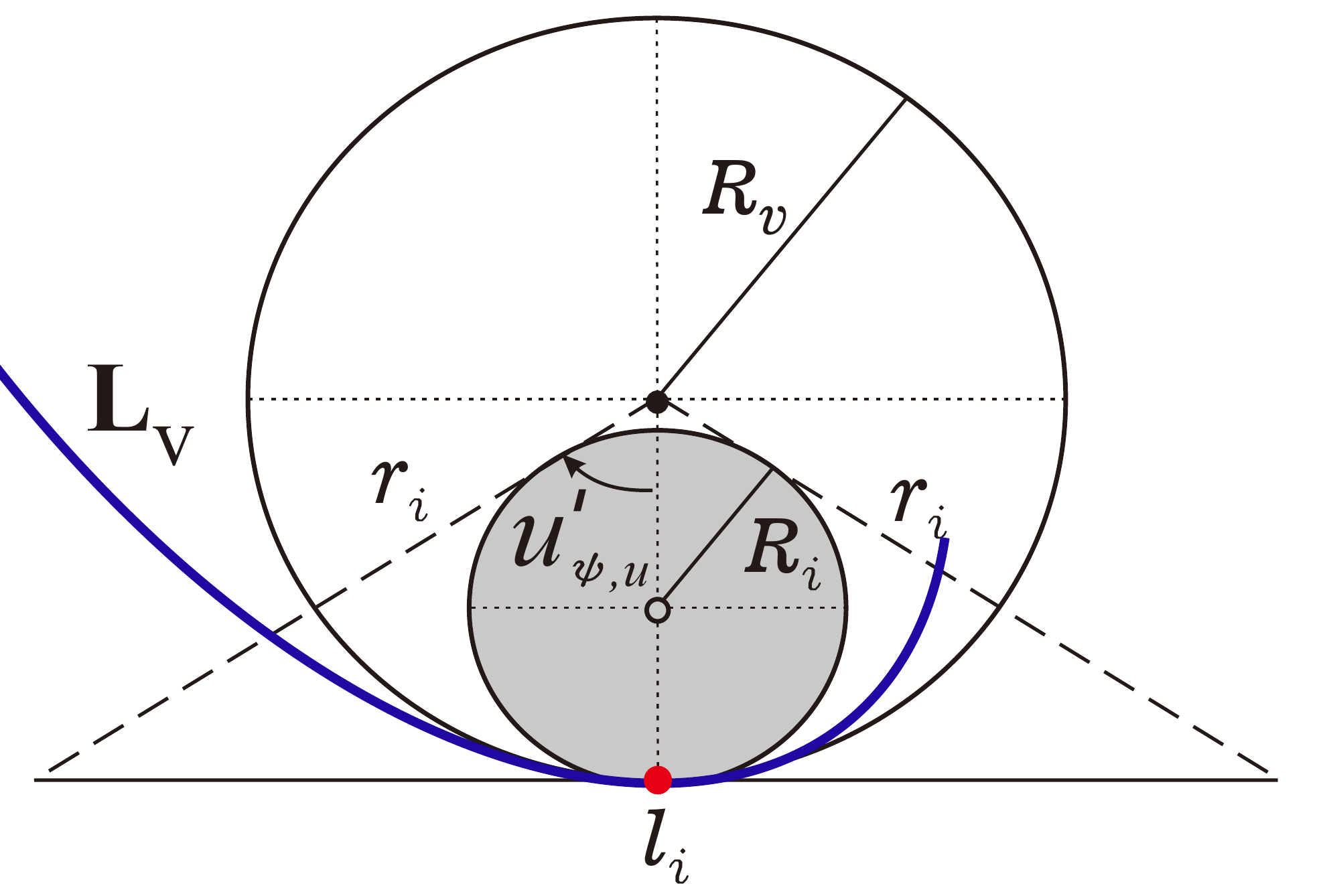}				
	\caption{\makehighlight{Radius curvature $R_i$ is calculated by incircle of $u'_{\psi,u}$ angle that determines the control inputs $\{\alpha_s,\gamma_s\}$.}}\label{Fig:QCOHeeightttt}
\end{figure}
\begin{prop}
	\label{Prop:ThefunctionsofGeometricController}
	By using the conditions of the proposition \ref{PropostionDarbouxControlinput}, functions of the geometric controller inputs (\ref{Eq:Geoemtricocontrolvector}) are designed based on sphere normal and geodesic curvatures $\{k_n,k_g\}$ with $C^{\infty}$ smoothness under safety constraints in Section \ref{Sec:ProblemStatement} as follows \footnote{Please refer to the proof of the Proposition \ref{Prop:ThefunctionsofGeometricController} to check what is the function of each variable in Eqs. (\ref{Eq:Thegammafunction})-(\ref{Eq:Thedeltafunction}).}
	\begin{eqnarray}
	\label{Eq:Thegammafunction}
	\gamma_s&=&1/R_i(u'_{\psi,u}), \\
	\label{Eq:ThealphaControlfunc}
	\alpha_s&=&\tan\left[\zeta\left(u'_{\psi,u},\psi(t) \right)\right]/R_i(u'_{\psi,u}),\\
	\delta &=&  \lambda_s(u_{v,u},v_v(t))/\lambda_t,
	\label{Eq:Thedeltafunction}
	\end{eqnarray}
	where $R_i$, $\zeta$, $\lambda_s$ and $\lambda_t$ are the changing imaginary sphere's radius, the angle of projected area due to steering difference, the arc-length and time step differences based on vehicle velocity, respectively. \makehighlight{Additionally, the stable motion convergence of the shared control $\uvec{u}$ is achieved under bounded inputs of assistive controller by 
		\begin{align}
&	\Vert \uvec{u}_c \Vert^2 \leq	(1-m)^2 \delta^2 \Big [ (1+R_v\gamma_s \left(\max{\left\{R_i\right\}}\right))^2 \nonumber\\
&+\alpha_s^2\left(\min{\left\{R_i\right\}},\max{\left\{\zeta\right\}}\right) \Big ]< \Vert \uvec{u}_b \Vert^2,
		\label{Eq:TransformNewInqControllerStableBound}
		\end{align}
		where $m=1/n$, $\uvec{u}_b$ is the user's intended exponentially convergent safe velocities that is expected by the robot to follow and $\Vert \cdot  \Vert$ is the Euclidean norm of the vector space.
}
\end{prop}
\begin{proof}
We have to design proper geometric functions to prescribe our new corresponding inputs $\{\delta,\alpha_s,\gamma_s\}$ depending on the expressed objectives. The arc-length-based input $\gamma_s$ in (\ref{EQ:TheCurvaturerelevantDif}) is the induced normal curvature of the virtual surface. We design $\gamma_s$ with the normal curvature of the rolling disc as $\gamma_s=1/R_i$ where $R_i$ is the varying curvature radius. From Fig. \ref{Fig:QCOHeeightttt}, the radius $R_i(u'_{\psi,u})$ is designed with the change of the incircle radius that is dependent on the steering angle of the user $u'_{\psi,u}$ as follows
 \begin{equation}
 R_i(u'_{\psi,u})=\frac{R_v}{\mu_r}+\left[\frac{(S_i-r_i)^2(S_i-l_i)}{S_i}\right]^{\frac{1}{2}},\;\;0 \leq |u'_{\psi,u}|<\frac{\pi}{2}
 \label{EQ:HeightQCO}
 \end{equation}
 where $S_i=(2r_i+l_i)/2$ is the area of encompassed triangle of the incircle, $\mu_r$ is the small value scaler to avoid singular point by the minimum curvature radius of $R_i$, and also $r_i=R_v/\cos u'_{\psi,u}$ and $l_i=2R_v\tan{u'_{\psi,u}}$ are adjacent and hypotenuse sides of isosceles triangle as shown in Fig. \ref{Fig:QCOHeeightttt}).

To design the arc-length-based input $\alpha_s$ in (\ref{Eq:UPSICgeometricontroler}), we have to formulate it with caution since $\delta \alpha_s$ correspond to the orientation velocity $\dot{\psi}$. This means we have to develop the steering input $u'_{\psi,u}$ with a constraining function to avoid any dangerous spinning. The wheeled mobile robot's rotation can easily become a danger for the user/patient who controls the vehicle. For example, the vehicle can topple in high-velocity maneuver if the direction does not get assigned correctly.

 The arc-length-based input $\alpha_s$ stands for the geodesic curvature of the virtual surface with radius $R_v$ of virtual wheel in (\ref{EQ:TheCurvaturerelevantDif}) \cite{TafrishiMAMT2021}, then, we use the geodesic curvature of an imaginary spherical surface in this controller as follows 
\begin{equation}
\alpha_s  = \tan(\zeta)/R_i,
\label{Eq:Alphageodesicarclength}
\end{equation}
where $\zeta$ is the projection angle (see Fig. \ref{Fig:GeodesicTorsioAngle} for the main idea).
\begin{figure}[t!]
	\centering	
	\vspace{3mm} %5mm vertical space
	\includegraphics[width=2.3 in]{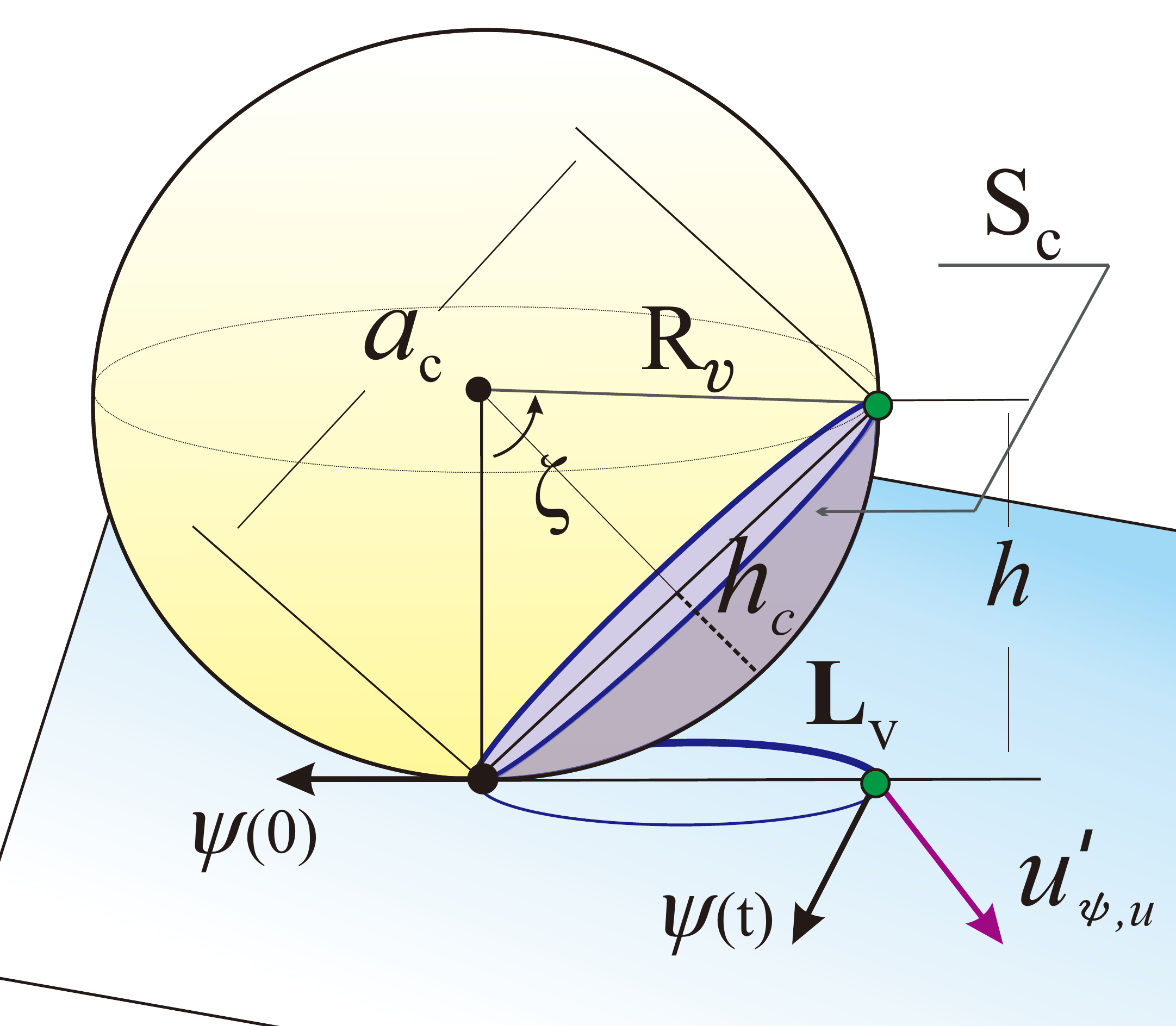}				
	\caption{\makehighlight{Computation of projection angle $\zeta$ based on the angular difference between user steer angle $u'_{\psi,u}$ and vehicle's orientation $\psi$.}}\label{Fig:GeodesicTorsioAngle}
\end{figure}
A trigonometric relation as shown in Fig. \ref{Fig:GeodesicTorsioAngle} is computed between the project angle $\zeta$ and steering angle $u'_{\psi,u}$ that is given from user's the joystick. Next, it is assumed that the angle $\zeta$ is the projection angle of constructed under-cap area $S_c$ that is found from
\begin{equation}
\zeta= 2 \tan^{-1}\left( a_c / 2 R_v \right),
\end{equation} 
 where $a_c$ is base diameter of under-cap area. We calculate the base diameter $a_c$ of the cap with the following formulation 
\begin{align}
 a_c=2\left[(S_c/\pi)-h^2_c\right]^{\frac{1}{2}},\;\;\;h_c=S_c/(2\pi R_v),
 \label{Eq:Finalcapparameters}
 \end{align}
where $S_t$ and $h_c$ are the total area of the cap-shaped region (blue region in Fig. \ref{Fig:GeodesicTorsioAngle}) and the height of the cap area. Next, we construct a relation between the steering angle and under-cap area $S_c$. We utilize the Gauss-Bonnet theorem \cite{DIfgeometry1976} while the imaginary sphere touches the circular closed blue path $\uvec{L}_v$ in Fig. \ref{Fig:GeodesicTorsioAngle} as follows
\begin{equation}
\Delta \psi= u'_{\psi,u}-\psi(t)=\iint_{{S}_{c}} \kappa_{o} \; dS_c = S_c/R_v^2,
\label{Eq:Gauss-Bonnetbase}
\end{equation} 
where $\kappa_{o} =1/R_v^2$ is the Gaussian curvature and $\psi(t)\in 2\pi$ is the current spin angle. The Gauss-Bonnet equation (\ref{Eq:Gauss-Bonnetbase}) gives us the ability to built a sphere under-cap area based on the steering angle given by raw steering input $u'_{\psi,u}$ and the current orientation of the vehicle $\psi(t)$. With re-ordering (\ref{Eq:Gauss-Bonnetbase}), our under-cap area becomes 
\begin{equation}
S_c=R_v^2\left | u'_{\psi,u}-\psi(t) \right|.
\label{Eq:FinalformulaguessSc}
\end{equation} 
The derived Eq. (\ref{Eq:FinalformulaguessSc}) is substituted back to Eq. (\ref{Eq:Finalcapparameters}). It can be seen that the spherical curvature helps to track the angular difference with a smoother function and transform it to angular velocity by $\delta \alpha$ since the under-cap area would have some values with respect to virtual wheel radius $R_v$.

Finally, we design the rolling rate $\delta$ function with the objectives from the problem statements. In high vehicle velocities, the geometric functions should not create a large level of deviations depending on the inputs e.g., the steering angle input. This can make the vehicle unstable. Note that this issue is much more challenging as the distance between the main actuating wheel and steering wheel $l$ becomes smaller; hence, we define rolling rate $\delta$ by
\begin{equation}
\delta= ds/dt \triangleq   \lambda_s/\lambda_t,
\end{equation}
where $\lambda_s$ and $\lambda_t$ are the arc-length and time step differences. We design the arc-length difference $\lambda_s$ in function of vehicle velocity ( Fig. \ref{Fig:VelocityvsArclength1} shows the example output of the function) as follows
\begin{equation}
\lambda_s (u_{v,u},v_{v})=
\begin{cases}
& u_{v,u}  \;\;\;\;\;\;\;\;\;\;\;\;\;\;\;\; \;\;\;\;\;\;\;\;\;\;\;\;\;\;\;v_{v}(t-1) \leq v_c \\
& u_{v,u} \cdot  e^{-\left[v_{v}(t-1)-v_c\right]/T}  \;\;\;\;v_{v}(t-1) > v_c
\end{cases}
\end{equation} 
where $v_c=v_m(1-\sigma)$ and $T$ are the critical velocity of the vehicle and the time scale, where $\sigma$ is critical velocity ratio with $38\%$ value, in here. These $v_c$ and $T$ values can change depending on the vehicle properties, or user preference.

\makehighlight{We check the controller stability by the assumption that the resultant control $\uvec{u}$ is always bounded with respect to user inputs under safety conditions. Let us consider the following constraint for stability of the controller 
\begin{equation}
\Vert \uvec{u} \Vert ^2 < \Vert \uvec{u}_b \Vert^2 \approx a^2_{\rho} \Vert  \Vert \uvec{u}_u(\tau) \Vert^2 \left(1-e^{-b_{\rho}\tau}\right)^2
\label{Eq:GeneralBoundedRule}
\end{equation}
where $a_{\rho}>0$, $b_{\rho}>0$  are constants and $\tau=t_n-t_{m}$ is a certain time interval. Then, by substituting Eq. (\ref{Eq:theControllerMainDefi}) into (\ref{Eq:GeneralBoundedRule}) and doing a transformation, we get
\begin{equation}
m^2  \Vert \uvec{u}_u \Vert^2+ 2m(1-m)\Vert\uvec{u}_u\Vert \Vert\uvec{u}_c\Vert+(1-m)^2 \Vert\uvec{u}_c\Vert< \Vert \uvec{u}_b \Vert^2
\label{Eq:TransformNewInq}
\end{equation}
where $m=1/n \in (0,\; 1]$. With knowing always $n>1$, the first two terms will always be smaller that third term $(1-m)^2\Vert\uvec{u}_c\Vert$; hence, the geometric controller determines whether the system is stable. Then, substituting the controller (\ref{Eq:Thegammafunction})-(\ref{Eq:ThealphaControlfunc}) into (\ref{Eq:TransformNewInq}) gives
\begin{equation}
(1-m)^2 \delta^2 \left[ (1+R_v\gamma_s)^2+\alpha_s^2 \right]< \Vert \uvec{u}_b \Vert^2
\label{Eq:TransformNewInqController}
\end{equation}
We know always $\underset{v_v \to \infty}{\lim}\delta^2=0 $ which grants this term is bounded. Also, based on the controller function in (\ref{Eq:Thegammafunction})-(\ref{Eq:ThealphaControlfunc})
the controller velocity is bounded by the given angle by user steering angle $u'_{\psi,u}\in[-\pi/2,\pi/2]$ which results %with bounded $R_i$ 
\begin{equation}
\gamma_s^2(s)<\gamma_s^2\left(\max{\left\{R_i\right\}}\right), \;\alpha_s^2(s)< \alpha_s^2 \left(\min{\left\{R_i\right\}},\max{\left\{\zeta\right\}}\right)
\end{equation}
where $\max{\left\{\zeta\right\}}<\pi/2$ based on the constraint on the difference between vehicle and user's angles, $|u'_{\psi,u}-\psi|<\pi/2$. This similarly grants the second term of (\ref{Eq:TransformNewInq}) is bounded that results to have a stable system with the bounded geometric controller by (\ref{Eq:TransformNewInqControllerStableBound}). }
\end{proof}

 \begin{figure}[t!]
	\centering
	\vspace{3mm} %5mm vertical space	
	\includegraphics[width=2.5 in]{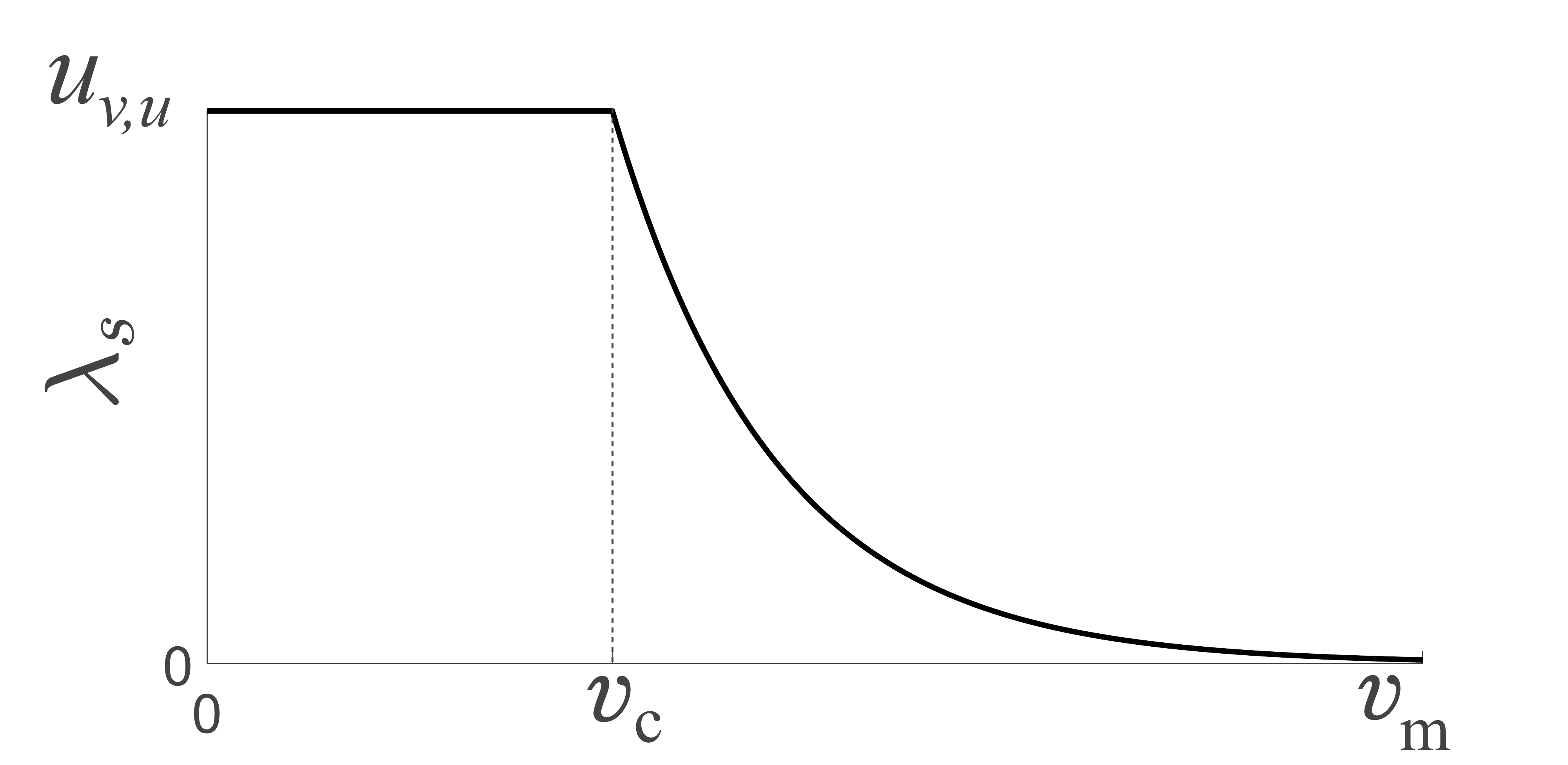}				
	\caption{\makehighlight{ The rolling rate $\delta$ design based on the arc-length differences on the vehicle velocity.}}\label{Fig:VelocityvsArclength1}
\end{figure}
\begin{rem}
If we check Eq. (\ref{Eq:UWCgeometriccontroler}) and Eq. (\ref{EQ:HeightQCO}), we can see that when the user gives large steering angles $u'_{\psi,u}$, the curvature radius $R_i$ becomes larger that makes the arc-length-based input $\gamma_s$ smaller (to decrease the overall vehicle velocity input $u_v$) that satisfies our safety constraint at (iii) in Section III. 	
\end{rem}
\begin{rem}
	The arc-length-based input $\alpha$ in (\ref{Eq:Alphageodesicarclength}) together with Eqs. (\ref{Eq:Finalcapparameters})-(\ref{Eq:FinalformulaguessSc}) shows that $u_{\omega}$ is constrained by under safety constraints. This function gives our second safety constraint, for condition (iv) in Section III, that the further we have large differences between vehicle angle $\psi(t)$ and user-assigned angle $u'_{\psi,u}$, the larger $\alpha_s$ can be while rolling rate limits the angular steering velocity by $\delta$. 
	% Note that we have achieved these naturally raised safety properties by using the differential geometry without any complex algorithm. 
\end{rem}
\begin{rem}
	The rolling rate $\delta(u_{v,u},v_v(t),v_m)$ of the virtual wheel that work as the controller is constraining both virtual surface inputs $(\alpha_s,\gamma_s)$ dependent on the maximum velocity allowance $v_m$ for the safety. 
\end{rem}
\begin{rem}
	\makehighlight{The inequality (\ref{Eq:TransformNewInq}) shows that larger shared control variable $n$ grants bounded inputs $\uvec{u}$ for a stable motion but very large values of $n$ will remove the user direct control/dominance from the shared control which tells there should be an upper limit for $n$. }
\end{rem}
\section{Experimental Studies} 
In this section, we experiment with our proposed assistive geometric controller.  We consider our controller with different subjects extensively with different routes.

In this experimental study, we consider two routes as projected Viviana's (infinity) curve and sharp spiral curve. In each experimental route, we have done the study by six participants with minimum knowledge of the wheelchair system shown in Fig. \ref{Fig:blockdiagram}. However, note that to clarify more interesting results for the lateral route i.e., sharp spiral curve, half of the participants were those who already took part in the Viviani's curve route. Note that due to the COVID-19 pandemic, we did follow all the disinfection and safety rules for the participant\footnote{Before the experiment, our research purpose, method and data handling were fully explained for participants, and we obtained their informed consent.}. \makehighlight{The participants were healthy with an average age of 26. The participants were expected to handle wheelchair with different velocities from low to high in these challenging curves.}
We used an electric wheelchair system that is commercially available as Whill Model C. The actuating wheels are located in $l=0.5$ m distance in this wheelchair. Also, our controller has the virtual wheel with a radius equal to the wheelchair wheel $R=R_v=0.133$ m. As the ground truth for accurate tracking of the wheelchair robot, we are using the HTC Vive Motion Tracker but none of the data is utilized by the controller. In these experiments, we change our vehicle maximum velocity between following range $v_m \in [1$-$3.5]$ m/s. \makehighlight{In this work, the time scale $T$ and the amplitude of time step difference $\lambda_t$ for (\ref{Eq:Thedeltafunction}) are chosen $2$ and $25$, respectively.} Each route was experimented by 6 participants where they tried 5 different maximum velocities between the following $v_m =\{1,2,2.5,3,3.5\}$ ranges with and without controller (refer to video for example result). We have asked each participant to take the NASA task load index (TLX) test after each run \cite{hart1988development}. The NASA-TLX is a multidimensional assessment tool that can help participants to rate their mental and physical experience. This test has 6 different scales: mental demand, frustration, effort, performance, physical demand and temporal demands. Please refer to the Youtube video for the example experiment done by using the proposed controller \cite{SAtafrishiYoutubeICRA2022}.

\begin{figure}[t!]
	\centering
	\vspace{3mm} %5mm vertical space	
a)	\includegraphics[width=2.1 in, height= 2.2 in]{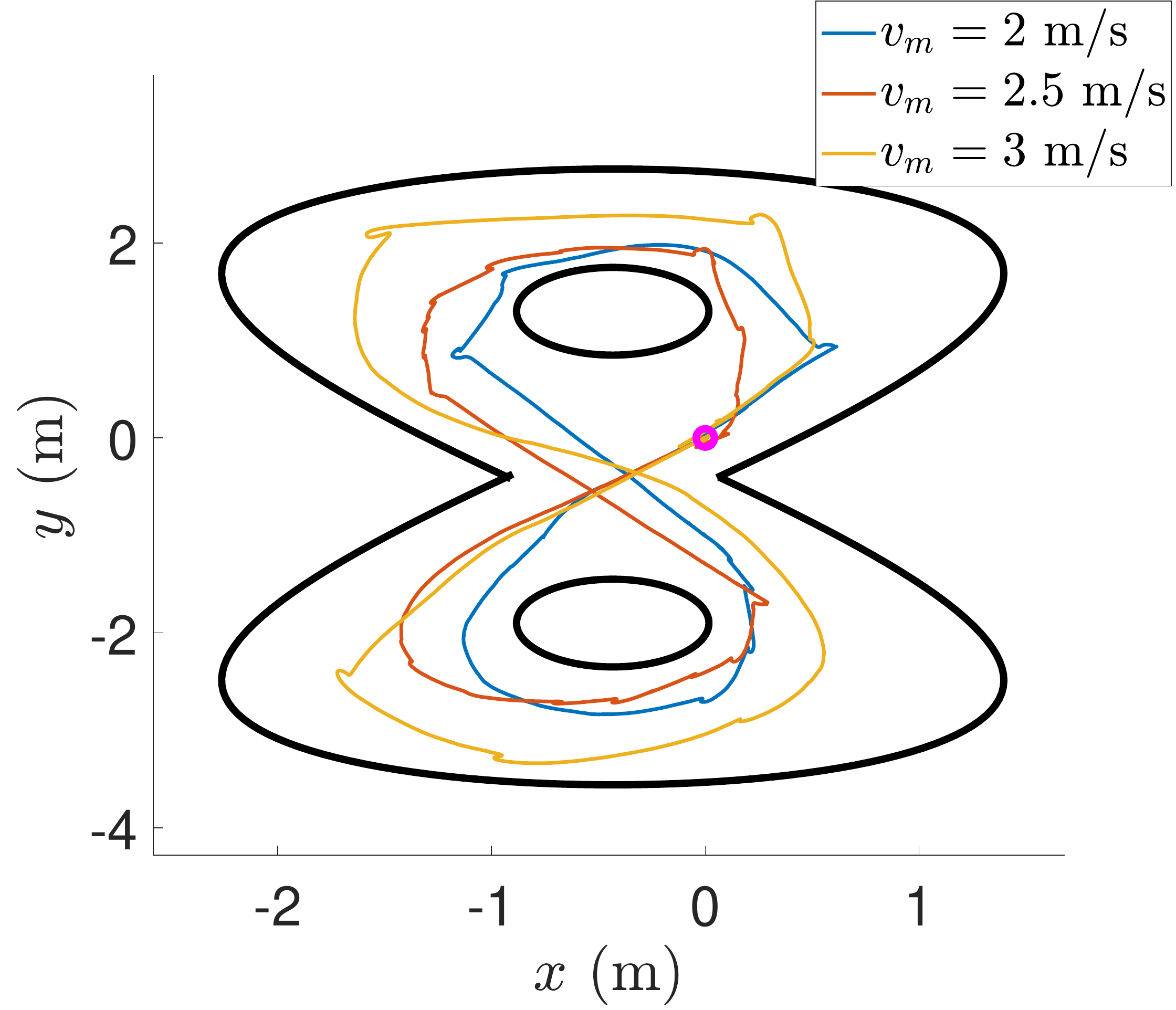}\\
b)	\includegraphics[width=2.1 in, height= 2.2 in]{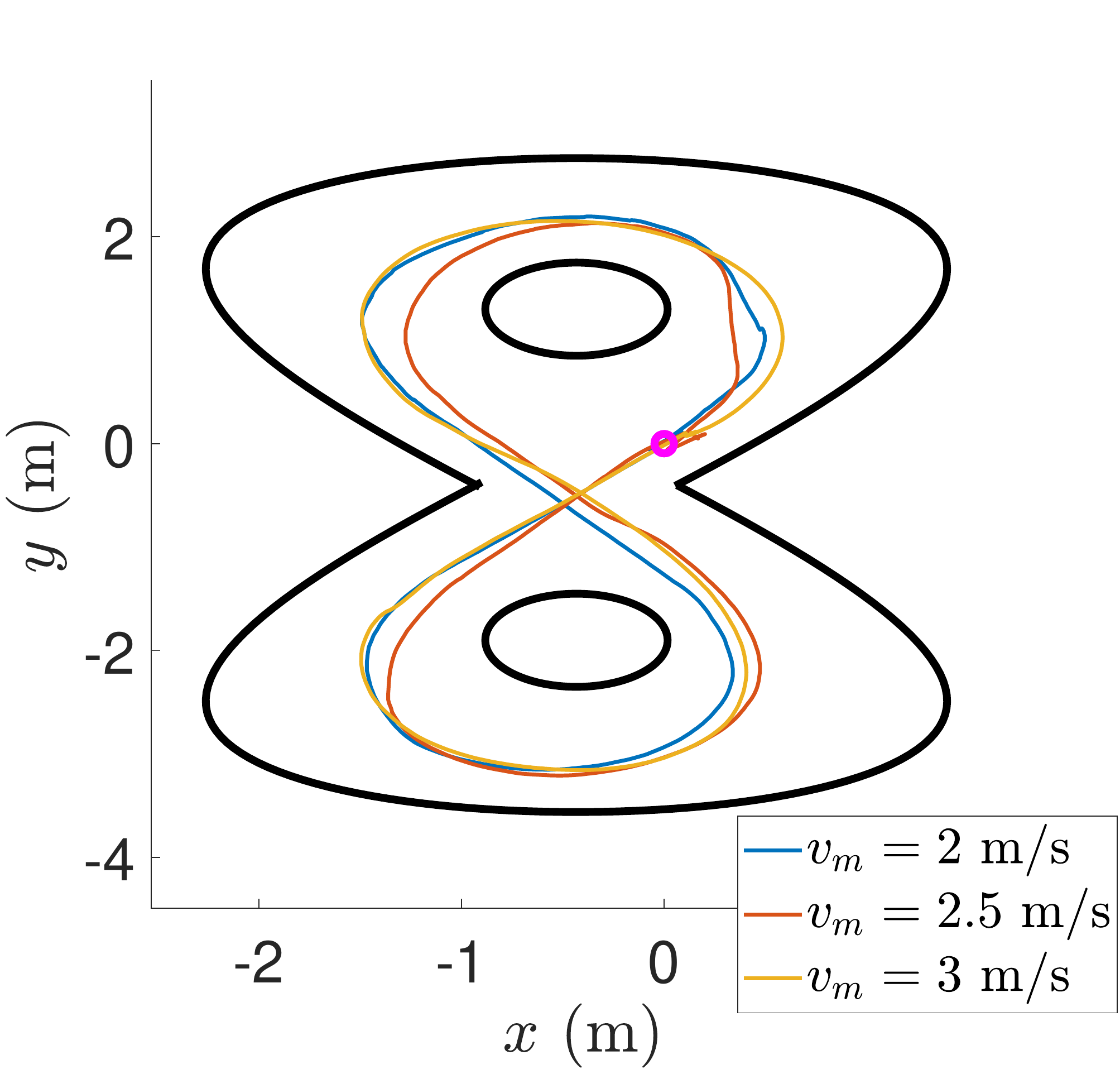}	
%\hspace{1.2 in}	 (a) \hspace{1.5 in}		(b) \hspace{.2 in}	
	\caption{The trajectories of an example subject on following the Viviani's curve (a) Without the controller (b) With the proposed controller.}\label{Fig:mapplotviviani}
\end{figure}
\begin{figure}[t!]
	\centering
	%\vspace{3mm} %5mm vertical space	
	\includegraphics[width=3.2 in, height= 1.7 in]{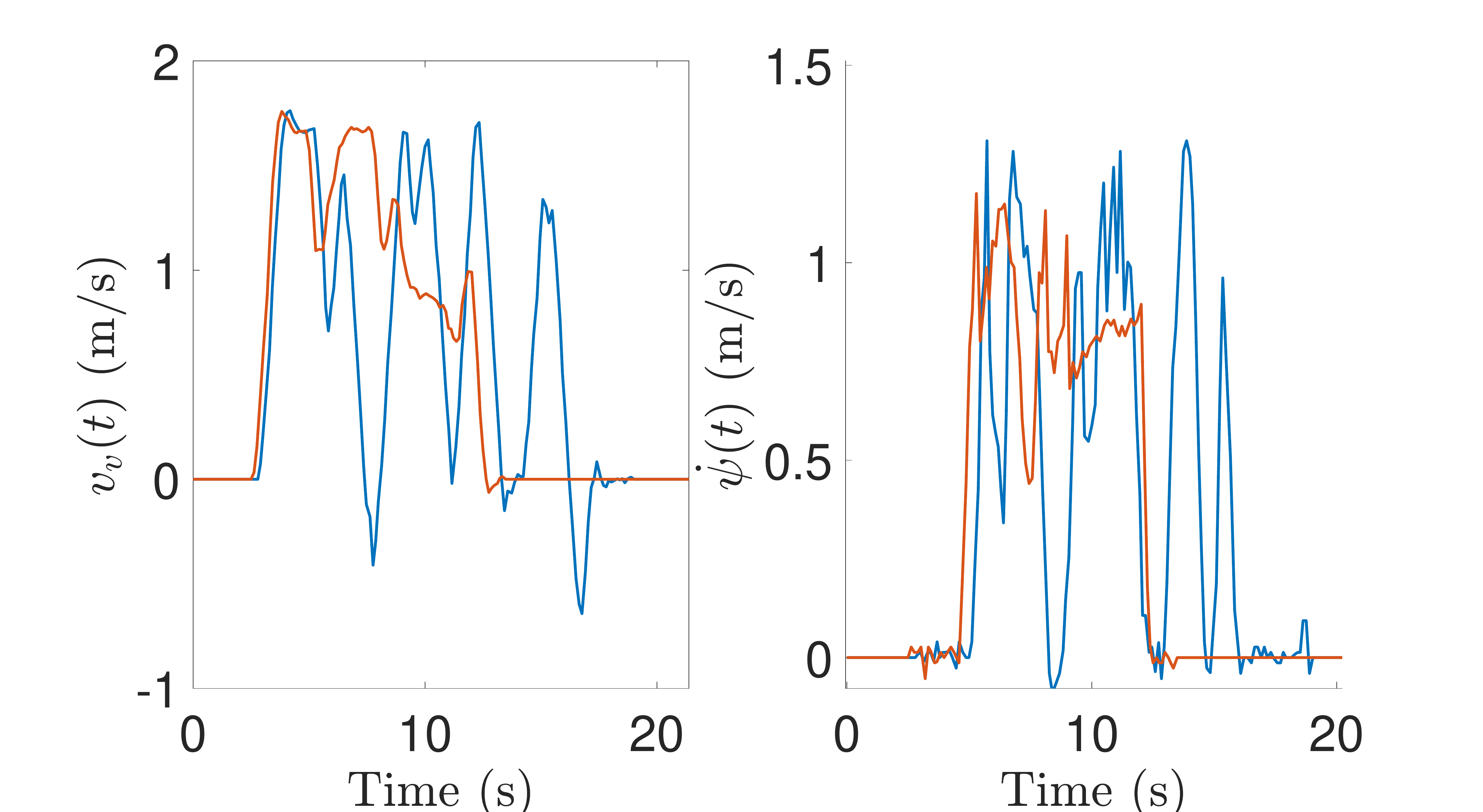}\\
	\includegraphics[width=3.2 in, height= 1.7 in]{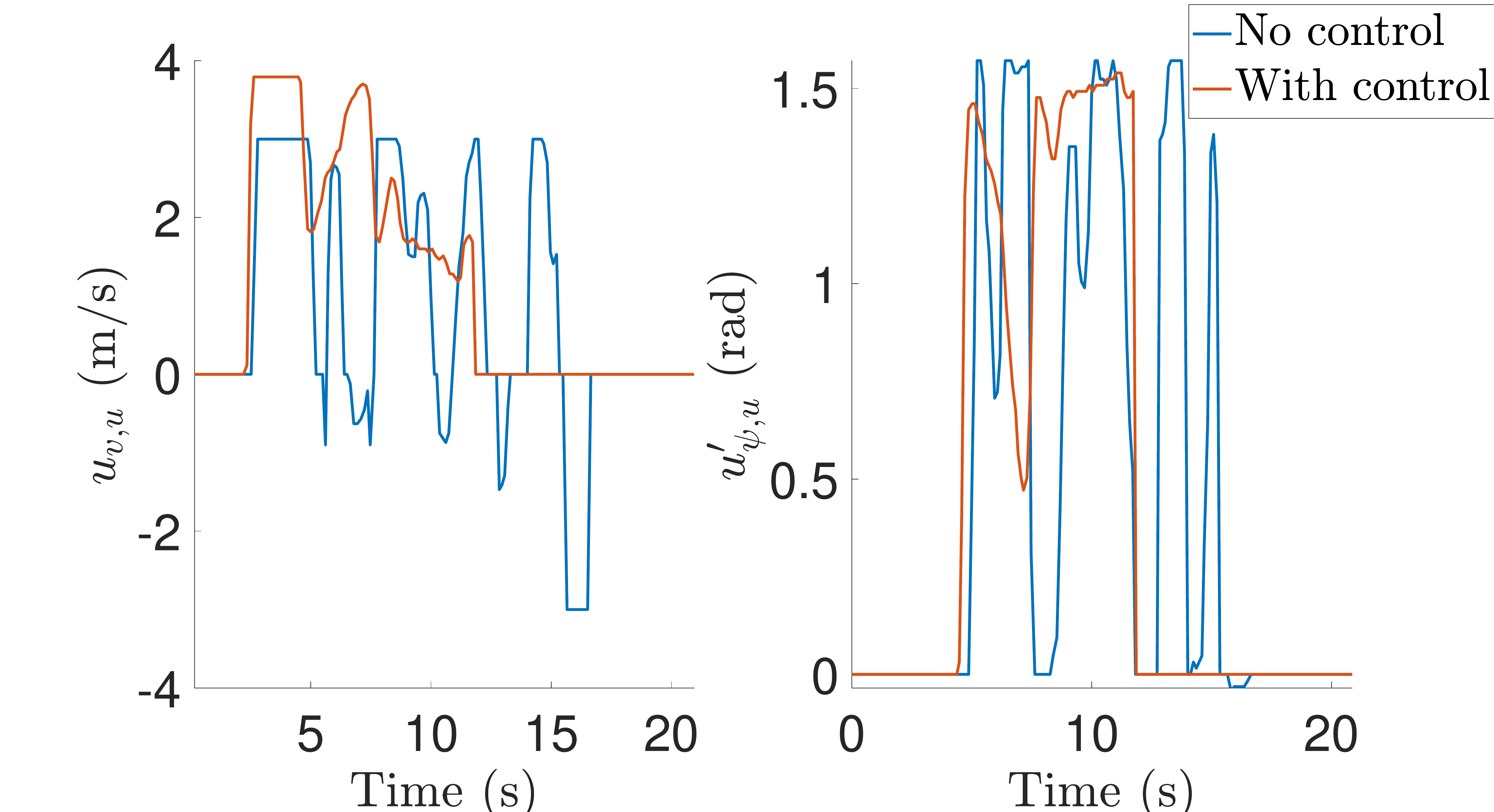}				
	\caption{The wheelchair robot velocities and user inputs for a Viviani's curve with $v_m=3$ m/s speed.}\label{Fig:Veocityinputspiral}
\end{figure}
\begin{figure}[t!]
	\centering
	%\vspace{3mm} %5mm vertical space	
a)	\includegraphics[width=2.3 in, height= 2.2 in]{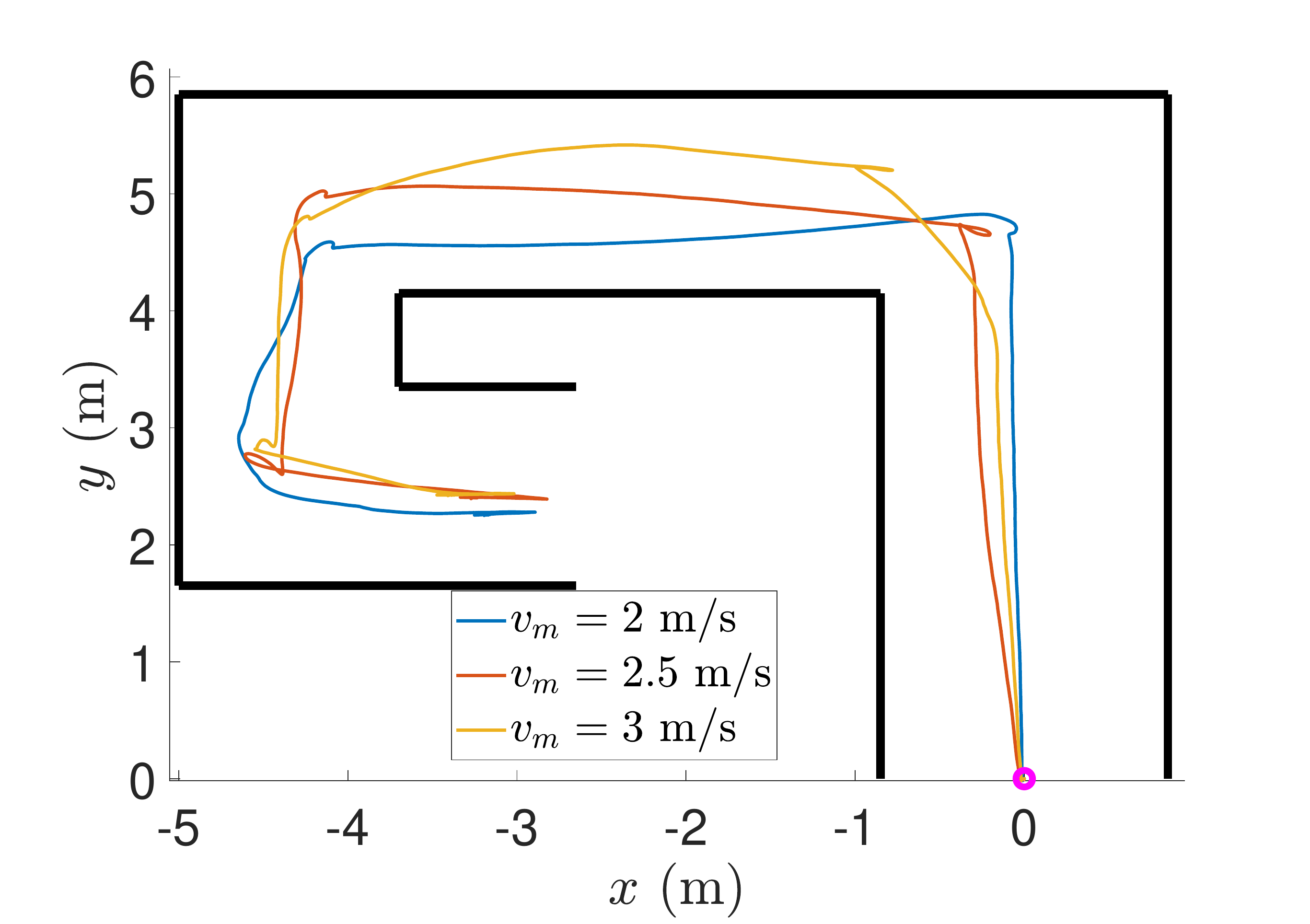}	\\
b)	\includegraphics[width=2.3 in, height= 2.2 in]{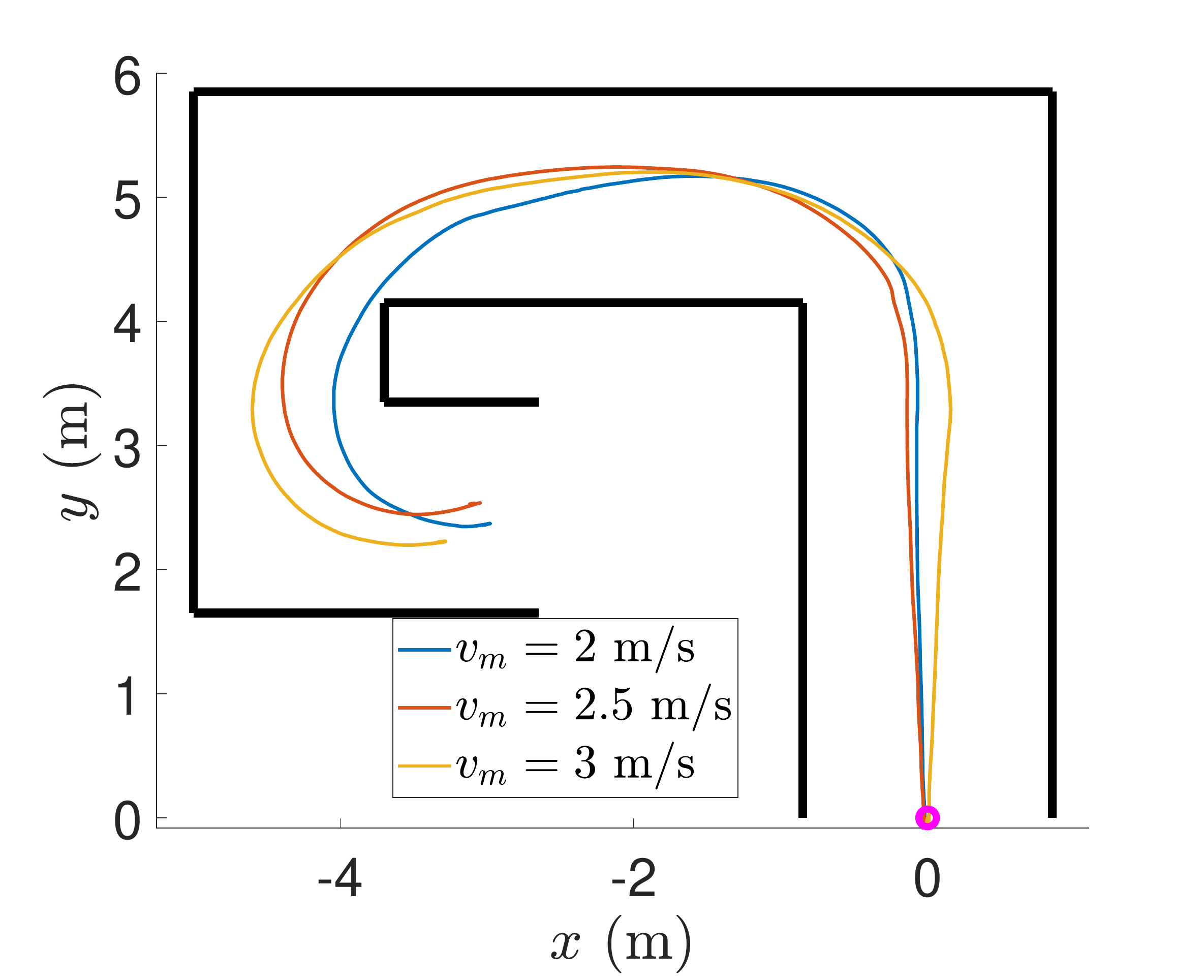}	
%\hspace{1.2 in}	 (a) \hspace{1.5 in}		(b) \hspace{.2 in}				
	\caption{The trajectories of an example subject on following the sharp spiral curve (a) Without the controller (b) With the proposed controller.}\label{Fig:mapplotspiral}
\end{figure}

In the first experiment, we choose Viviani's curve as the route. Fig. \ref{Fig:mapplotviviani} illustrates example results of a participant for different maximum velocities with and without the controller. We can see that the user has many deviations with sharp and unnatural maneuvers on the traversed path when there is no controller. This issue is clearer when the velocity is higher, $3$ m/s. However, our proposed assistive controller can keep the person in a smooth trajectory along the path. In order to understand how effective is the proposed controller, we can check the vehicle velocities $\uvec{v}=[v_v,\dot{\psi}]$, and user inputs for example high velocity (3 m/s) shown in Fig. \ref{Fig:Veocityinputspiral}. The vehicle velocity with the controller is very smooth in contrast to without any controller. As it is clear from user inputs by the joystick $(u_{v,u}, u'_{\psi,u})$, the users have lots of fluctuations in their input controls since the vehicle increases the velocity very fast and it becomes challenging to assign proper steering angle by the joystick. Thus, the user should put lots of effort to keep the wheelchair in the desired trajectory. If we quantify the user effort by the number of times that joystick goes from high value to as low as zero (in particular the input velocity $u_{v,u}$), then, the user is at least 7-8 times more prone to put higher effort without any control.

\begin{figure}[t!]
	\centering
%	\vspace{3mm} %5mm vertical space	
		\includegraphics[width=3.2 in, height= 1.6 in]{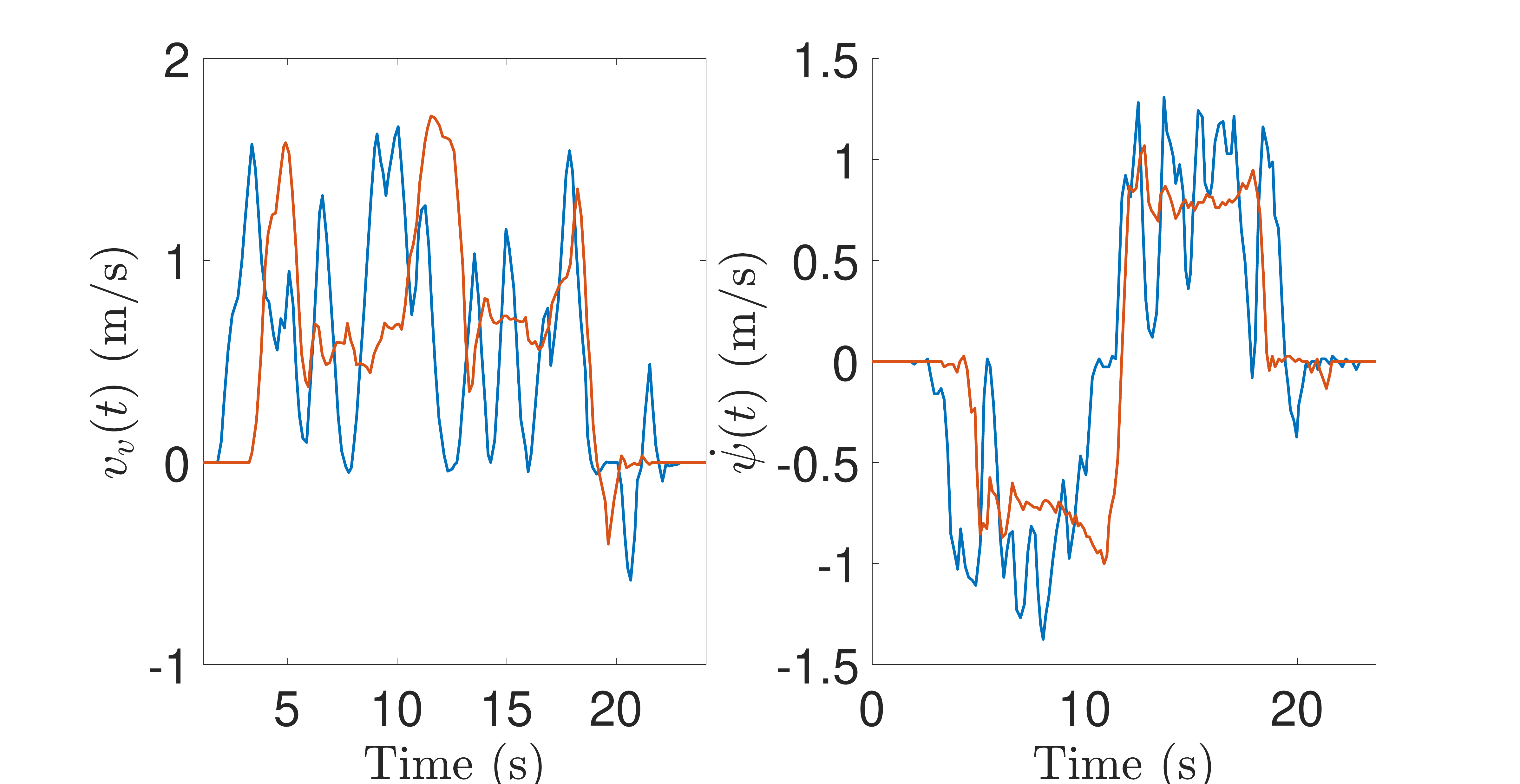}		\\
		\includegraphics[width=3.2 in, height= 1.6 in]{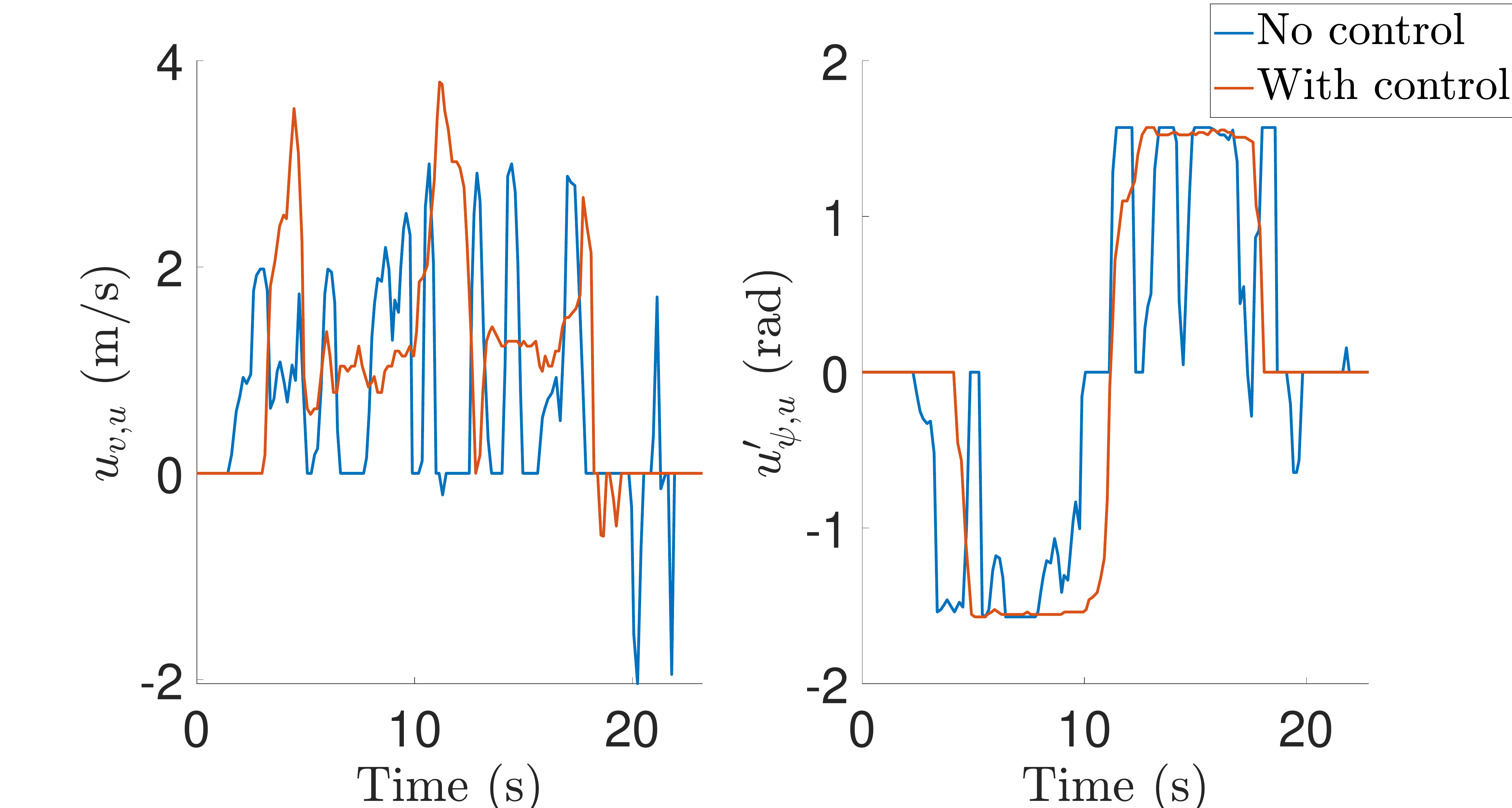}				
	\caption{The wheelchair robot velocities and user inputs for a sharp spiral curve with $v_m=3$ m/s speed.}\label{Fig:VeocityinputVivia}
\end{figure}

Another challenging route as the sharp spiral curve is chosen for evaluating the performance of the controller. Fig. \ref{Fig:mapplotspiral} illustrates the ground truth that is recorded for one of the participants. Similar to the trajectories of Viviani's curve, the vehicle follows smooth and non-fluctuating trajectories when the users are using an assistive controller, in particular during corner rotations. Interestingly, the proposed controller does not have any desired configuration but it can assist the user's inputs in a much smoother maneuver as shown in Fig. \ref{Fig:mapplotspiral}-b. It is important to note that this difference is more visible when velocity is higher $v_m>2$ m/s. The vehicle velocities $\uvec{v}=[v_v,\dot{\psi}]$ as shown in Fig. \ref{Fig:VeocityinputVivia} are smoother with minimum fluctuations by utilizing the controller where it is followed by the user inputs properly. Similarly, the user effort to make the wheelchair follow the proper path is about 8 times easier based on the number of times, he/she should go from maximum to minimum input [see Fig. \ref{Fig:VeocityinputVivia}-b]. 
 
\begin{figure}[t!]
 	\centering
 	\vspace{3mm} %5mm vertical space	
a) 	\includegraphics[width=2.7 in, height= 1.9 in]{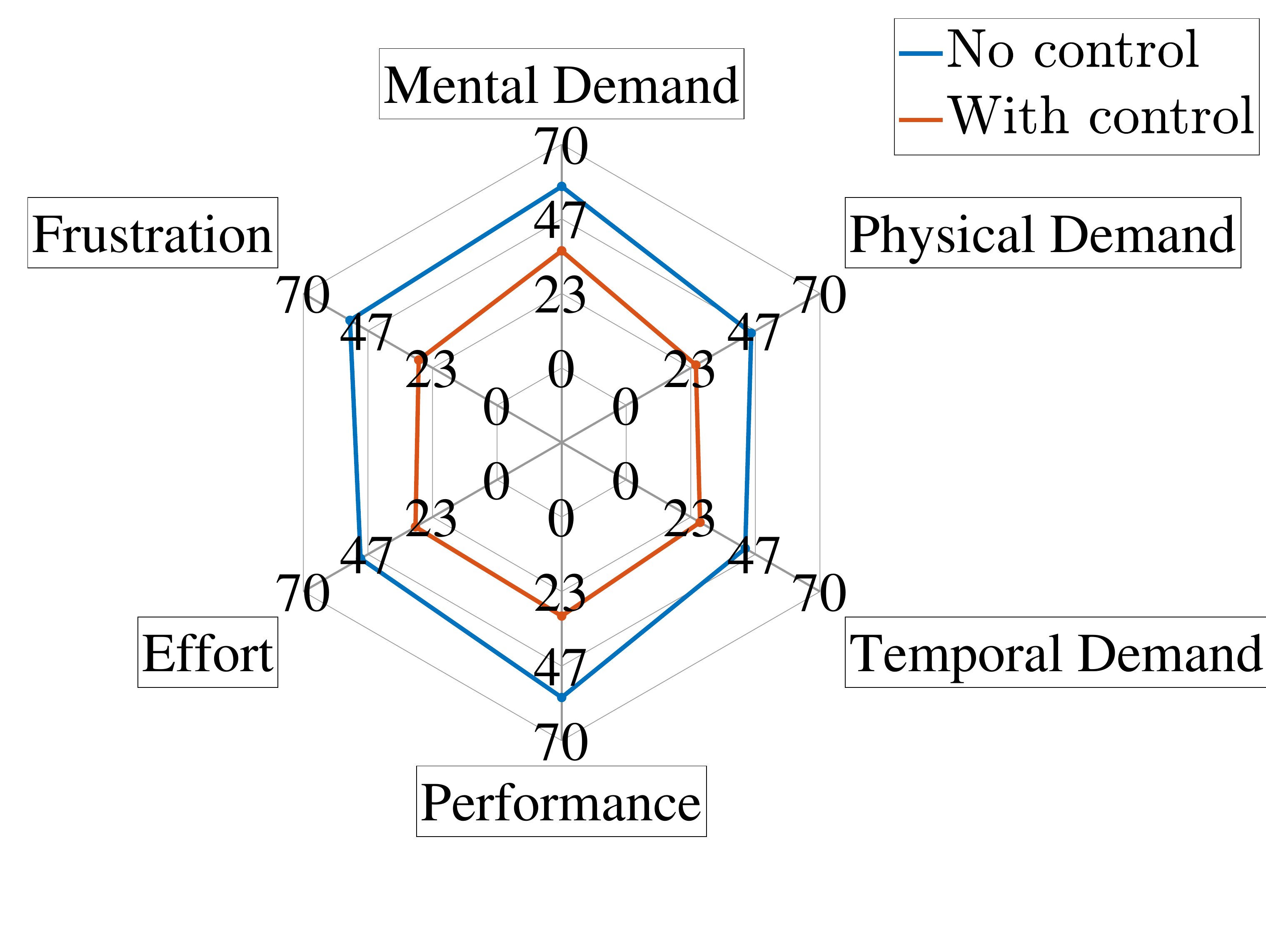}	\\
b)	\includegraphics[width=2.7 in, height= 1.9 in]{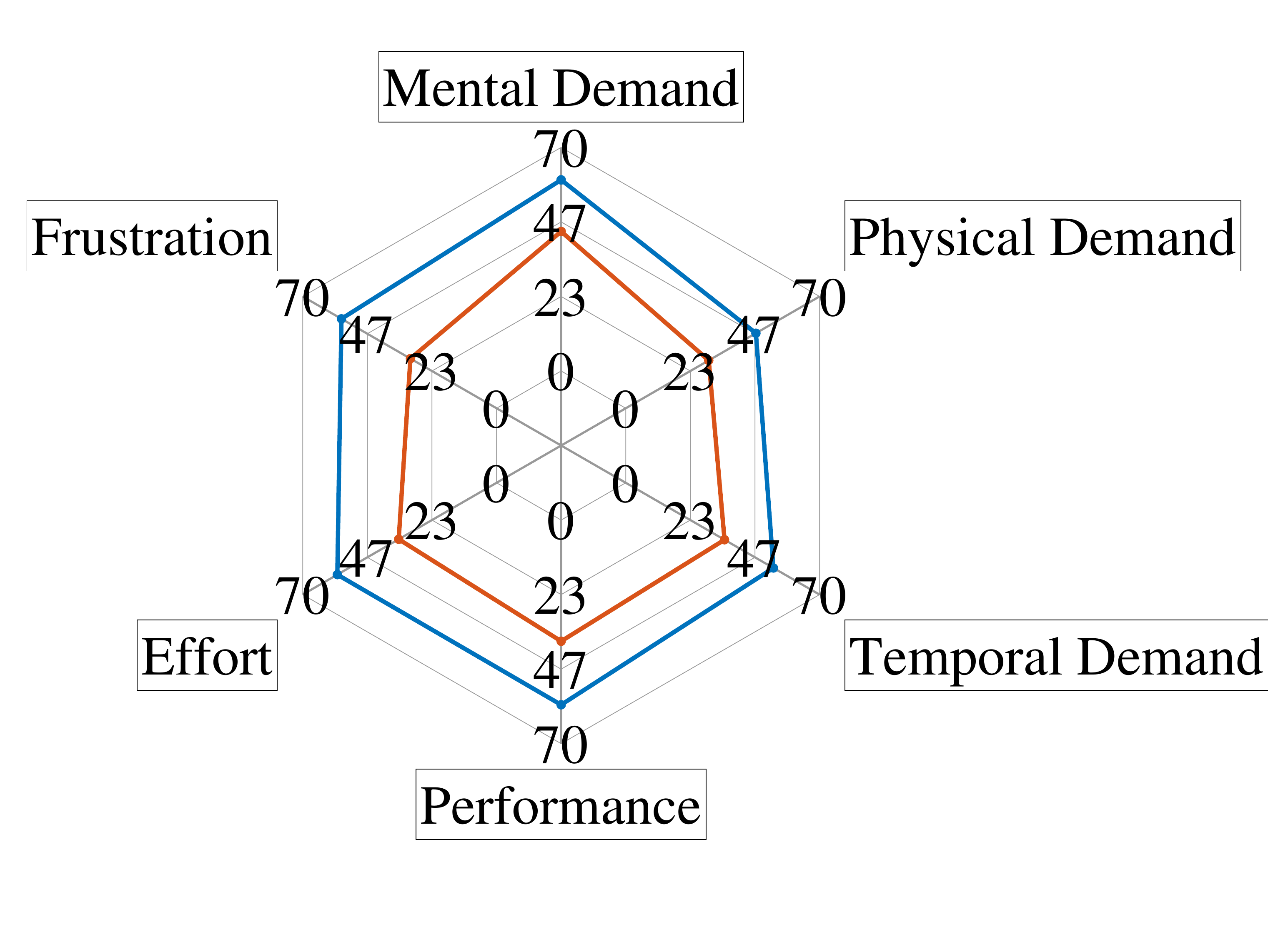}	
%\hspace{1.2 in}	 (a) \hspace{1.5 in}		(b) \hspace{.2 in}			
 	\caption{NASA-TLX results with raw ratings (a) The Viviani's curve (b) Sharp spiral curve.}\label{Fig:spiderplot}
 \end{figure}
 \begin{figure}[t!]
	\centering
	\vspace{3mm} %5mm vertical space	
a)		\includegraphics[width=2.7 in, height= 2.0 in]{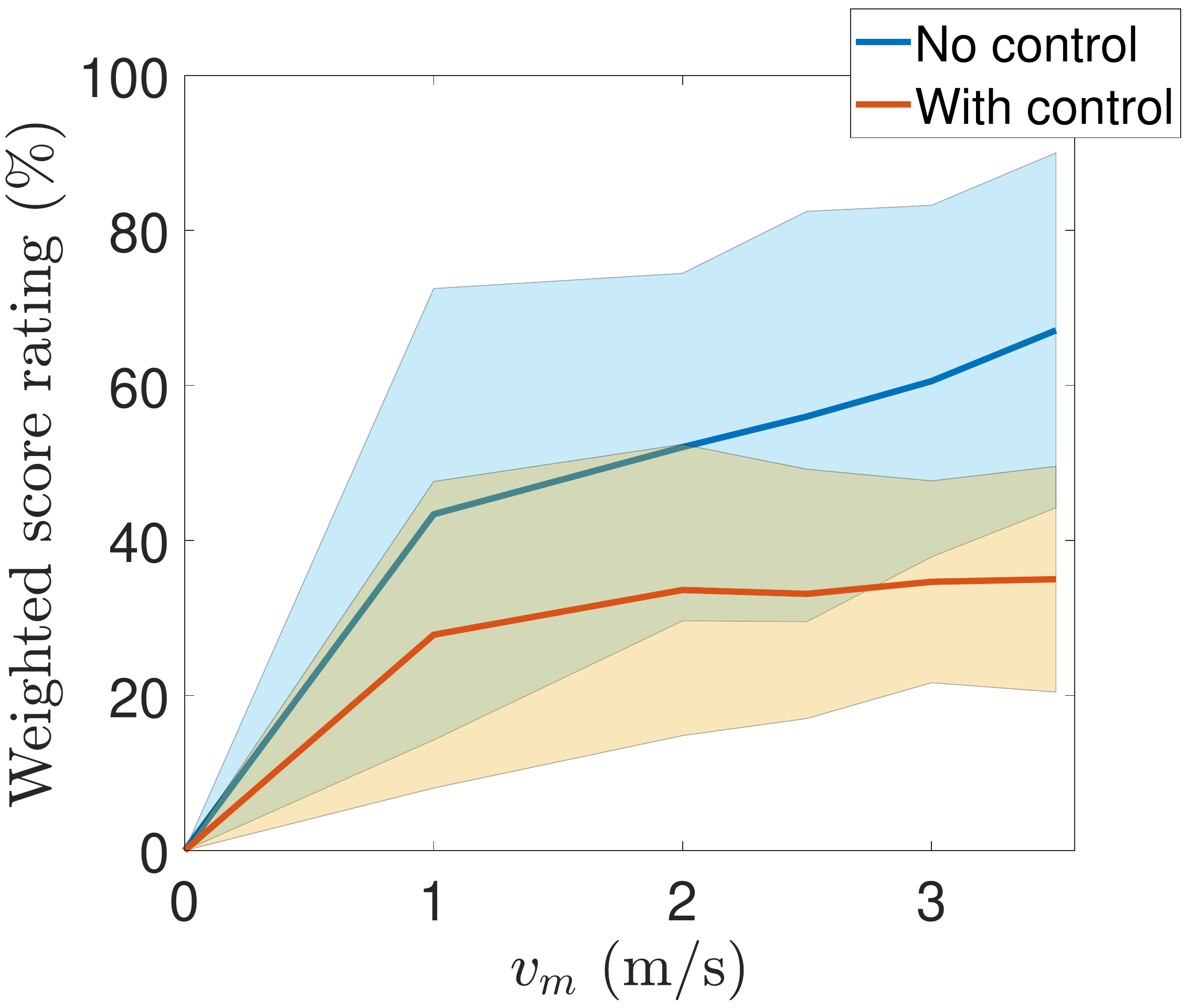}\\		
b)		\includegraphics[width=2.7 in, height= 2.0 in]{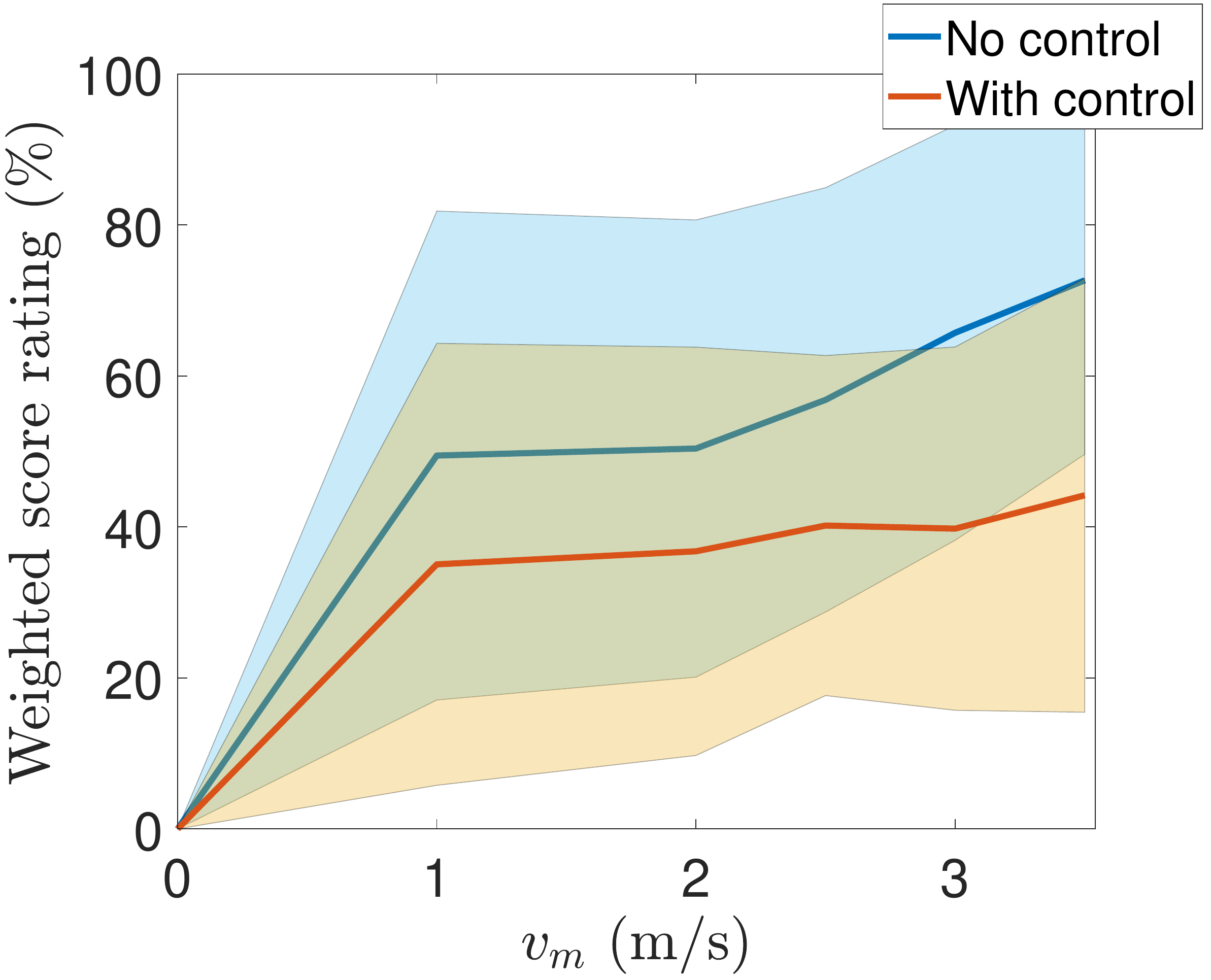}	
%\hspace{1.2 in}	 (a) \hspace{1.5 in}		(b) \hspace{.2 in}						
	\caption{\makehighlight{The weighted score rating with NASA-TLX questionnaire (a) The Viviani's curve (b) Sharp spiral curve. The solid line and shaded areas are the mean value and standard deviation. }}\label{Fig:ratioweighhed}
\end{figure}
Finally, to generalize our findings and understand how participants experience the run with and without the controller, we analyze the answered NASA-TLX questionnaire. At first, Fig. \ref{Fig:spiderplot} demonstrates the users' raw rating without weights. In this plot, the ratings are the mean of different velocities ($v_m \in [1$-$3.5]$ m/s) that participants have undergone the trials. It is clear from the graph that our proposed controller has a great satisfaction level (note that the lower value for each scale, the better it is) with an average difference of 21\% for both trajectories. Although the success of the proposed assistive controller can be verified, there are indications that users feel some challenges with respect to the mental demand and effort. We hypothesize that the source of high mental demand is the initial learning curve that should be followed to achieve better results by the controller. Thus, the participants feel more responsible in deciding the trajectory. \makehighlight{This means there are certain limitations of the current controller. For example, the following questions arise: how we can decrease the burden of high sensitivity to user inputs while following safe and smooth trajectories. Also, how we can sense user intention by extra sensors to minimize her/his mental demand.} We plan to develop proper strategies to answer these questions.

Another advantage of the controller can be seen with the weighted score rating in different maximum velocities as indicated in Fig. \ref{Fig:ratioweighhed}. It can be interpreted that the trend of the evaluation is similar to both of the routes. For the Viviani's curve, the rating score of the participants is around 29\% with the controller while the rating increase over 65 \% for the case of high velocities ($v_m=3.5$ m/s) without the controller. With a similar pattern for sharp spiral curves, the participants are approximately 20\% more satisfied with the use of the controller. It is clear that the satisfaction level gets higher as the controller mean score stays around 40\%. Also, the standard deviation of the scores is lesser which implies most of the evaluated participants give approximately similar score ratings for the experimented assistive controller. 
 
\section{Conclusions}
In this paper, we proposed a novel geometric controller for assisting the users of differential-drive wheeled mobile robots, without having any priori desired goal. The assistive controller was using a virtual wheel formulation based on the Darboux frame kinematics to correct the vehicle motion. At first, the kinematics of the differential-drive vehicle together with Darboux frame kinematics for the controller is described. Next, the motion control problem with considering the safety constraints is explained. Then, we developed our geometric assistive controller with defining the proper functions. Finally, the controller performance at different routes was evaluated. The experiment took place with different participants using a differential-drive wheelchair robot. We were able to confirm that the controller outperforms  with great level satisfaction from collected questionnaire (NASA-TLX) responses besides the motion behaviors. The results also show that users using the assistive controller put less effort in directing wheelchairs in complex paths with user-assigned vehicle velocities while the realized trajectories are smooth. \makehighlight{ Note that the experiments done with the wheelchair were in the maximum velocity of $3.5$ m/s. This shows how assistive controller is essential since most vehicles have even higher velocities in which emphasize the importance of assistive controllers.}

\makehighlight{In the future, we plan to develop more advanced shared autonomy by combining the assistive controller with path planning methods.} Also, we will extend the problem to more complex 3-dimensional trajectories to create a more generic assistive controller that can be applied to other robot platforms namely, manipulators, drones, and so on.   
\section{Acknowledgments}
We would like to thank Prof. Jean-Paul Laumond and anonymous reviewers for their valuable comments regarding the paper. This work was partially supported by JSPS KAKENHI grant number JP21K20391 and Japan Science and Technology Agency (JST) [Moonshot R$\&$D Program] under Grant JPMJMS2034. 
 
% and other anonymous reviewers 
%\addtolength{\textheight}{-12cm}   % This command serves to balance the column lengths
                                  % on the last page of the document manually. It shortens
                                  % the textheight of the last page by a suitable amount.
                                  % This command does not take effect until the next page
                                  % so it should come on the page before the last. Make
                                  % sure that you do not shorten the textheight too much.

%%%%%%%%%%%%%%%%%%%%%%%%%%%%%%%%%%%%%%%%%%%%%%%%%%%%%%%%%%%%%%%%%%%%%%%%%%%%%%%%

\bibliographystyle{IEEEtran}
\bibliography{MotionDarbWheel} 

% Generated by IEEEtran.bst, version: 1.14 (2015/08/26)
\begin{thebibliography}{10}
\providecommand{\url}[1]{#1}
\csname url@samestyle\endcsname
\providecommand{\newblock}{\relax}
\providecommand{\bibinfo}[2]{#2}
\providecommand{\BIBentrySTDinterwordspacing}{\spaceskip=0pt\relax}
\providecommand{\BIBentryALTinterwordstretchfactor}{4}
\providecommand{\BIBentryALTinterwordspacing}{\spaceskip=\fontdimen2\font plus
\BIBentryALTinterwordstretchfactor\fontdimen3\font minus
  \fontdimen4\font\relax}
\providecommand{\BIBforeignlanguage}[2]{{%
\expandafter\ifx\csname l@#1\endcsname\relax
\typeout{** WARNING: IEEEtran.bst: No hyphenation pattern has been}%
\typeout{** loaded for the language `#1'. Using the pattern for}%
\typeout{** the default language instead.}%
\else
\language=\csname l@#1\endcsname
\fi
#2}}
\providecommand{\BIBdecl}{\relax}
\BIBdecl

\bibitem{261508}
R.~M. {Murray} and S.~S. {Sastry}, ``Steering nonholonomic systems in chained
  form,'' in \emph{Proceedings of the 30th IEEE Conference on Decision and
  Control}, 1991, pp. 1121--1126.

\bibitem{badreddin1993fuzzy}
E.~Badreddin and M.~Mansour, ``Fuzzy-tuned state-feedback control of a
  non-holonomic mobile robot,'' \emph{IFAC Proceedings Volumes}, vol.~26,
  no.~2, pp. 769--772, 1993.

\bibitem{laumond1994motion}
J.-P. Laumond, P.~E. Jacobs, M.~Taix, and R.~M. Murray, ``A motion planner for
  nonholonomic mobile robots,'' \emph{IEEE Trans. Rob. Autom.}, vol.~10, no.~5,
  pp. 577--593, 1994.

\bibitem{Astolfi1999}
A.~Astolfi, ``{Exponential Stabilization of a Wheeled Mobile Robot Via
  Discontinuous Control},'' \emph{ASME J. Dyn. Syst. Meas. Control}, vol. 121,
  no.~1, pp. 121--126, 03 1999.

\bibitem{gasparetto2015path}
A.~Gasparetto, P.~Boscariol, A.~Lanzutti, and R.~Vidoni, ``Path planning and
  trajectory planning algorithms: A general overview,'' \emph{Motion and
  operation planning of robotic systems}, pp. 3--27, 2015.

\bibitem{zheng1993recent}
Y.-F. Zheng, \emph{Recent trends in mobile robots}.\hskip 1em plus 0.5em minus
  0.4em\relax World scientific, 1993, vol.~11.

\bibitem{laumond1998robot}
J.-P. Laumond \emph{et~al.}, \emph{Robot motion planning and control}.\hskip
  1em plus 0.5em minus 0.4em\relax Springer, 1998, vol. 229.

\bibitem{giordano2006nonholonomic}
P.~R. Giordano, M.~Vendittelli, J.-P. Laumond, and P.~Soueres, ``Nonholonomic
  distance to polygonal obstacles for a car-like robot of polygonal shape,''
  \emph{IEEE Trans. Robot.}, vol.~22, no.~5, pp. 1040--1047, 2006.

\bibitem{hsu2002randomized}
D.~Hsu, R.~Kindel, J.-C. Latombe, and S.~Rock, ``Randomized kinodynamic motion
  planning with moving obstacles,'' \emph{Int. J. Rob. Res.}, vol.~21, no.~3,
  pp. 233--255, 2002.

\bibitem{lavalle2006planning}
S.~M. LaValle, \emph{Planning algorithms}.\hskip 1em plus 0.5em minus
  0.4em\relax Cambridge university press, 2006.

\bibitem{minguez2016motion}
J.~Minguez, F.~Lamiraux, and J.-P. Laumond, ``Motion planning and obstacle
  avoidance,'' in \emph{Springer handbook of robotics}.\hskip 1em plus 0.5em
  minus 0.4em\relax Springer, 2016, pp. 1177--1202.

\bibitem{paden2016survey}
B.~Paden, M.~{\v{C}}{\'a}p, S.~Z. Yong, D.~Yershov, and E.~Frazzoli, ``A survey
  of motion planning and control techniques for self-driving urban vehicles,''
  \emph{IEEE Trans. Intell. Veh.}, vol.~1, no.~1, pp. 33--55, 2016.

\bibitem{badue2021self}
C.~Badue, R.~Guidolini, R.~V. Carneiro, P.~Azevedo, V.~B. Cardoso, A.~Forechi,
  L.~Jesus, R.~Berriel, T.~M. Paixao, F.~Mutz \emph{et~al.}, ``Self-driving
  cars: A survey,'' \emph{Expert Syst. Appl.}, vol. 165, p. 113816, 2021.

\bibitem{samson1995control}
C.~Samson, ``Control of chained systems application to path following and
  time-varying point-stabilization of mobile robots,'' \emph{IEEE Trans.
  Automat. Contr.}, vol.~40, no.~1, pp. 64--77, 1995.

\bibitem{morin2008motion}
P.~Morin and C.~Samson, ``Motion control of wheeled mobile robots.''
  \emph{Springer handbook of robotics}, vol.~1, pp. 799--826, 2008.

\bibitem{tzafestas2018mobile}
S.~G. Tzafestas, ``Mobile robot control and navigation: A global overview,''
  \emph{J. Intell. Robot. Syst.}, vol.~91, no.~1, pp. 35--58, 2018.

\bibitem{klancar2005mobile}
G.~Klancar, D.~Matko, and S.~Blazic, ``Mobile robot control on a reference
  path,'' in \emph{Proc. IEEE International Symposium on Mediterrean Conference
  on Control and Automation Intelligent Control}, 2005, pp. 1343--1348.

\bibitem{sousa2016trajectory}
R.~L.~S. Sousa, M.~D. do~Nascimento~Forte, F.~G. Nogueira, and B.~C. Torrico,
  ``Trajectory tracking control of a nonholonomic mobile robot with
  differential drive,'' in \emph{IEEE Biennial Congress of Argentina}, 2016,
  pp. 1--6.

\bibitem{forte2018reference}
M.~D. Forte, W.~B. Correia, F.~G. Nogueira, and B.~C. Torrico, ``Reference
  tracking of a nonholonomic mobile robot using sensor fusion techniques and
  linear control,'' \emph{IFAC-PapersOnLine}, vol.~51, no.~4, pp. 364--369,
  2018.

\bibitem{abdelwahab2020trajectory}
M.~Abdelwahab, V.~Parque, A.~M.~F. Elbab, A.~Abouelsoud, and S.~Sugano,
  ``Trajectory tracking of wheeled mobile robots using z-number based fuzzy
  logic,'' \emph{IEEE Access}, vol.~8, pp. 18\,426--18\,441, 2020.

\bibitem{prassler2001robotics}
E.~Prassler, J.~Scholz, and P.~Fiorini, ``A robotics wheelchair for crowded
  public environment,'' \emph{IEEE Robot. Autom. Mag.}, vol.~8, no.~1, pp.
  38--45, 2001.

\bibitem{dragan2013policy}
A.~D. Dragan and S.~S. Srinivasa, ``A policy-blending formalism for shared
  control,'' \emph{Int. J. Rob. Res.}, vol.~32, no.~7, pp. 790--805, 2013.

\bibitem{javdani2015shared}
S.~Javdani, S.~S. Srinivasa, and J.~A. Bagnell, ``Shared autonomy via hindsight
  optimization,'' \emph{Robot. Sci. Syst.}, vol. 2015, 2015.

\bibitem{javdani2018shared}
S.~Javdani, H.~Admoni, S.~Pellegrinelli, S.~S. Srinivasa, and J.~A. Bagnell,
  ``Shared autonomy via hindsight optimization for teleoperation and teaming,''
  \emph{Int. J. Rob. Res.}, vol.~37, no.~7, pp. 717--742, 2018.

\bibitem{park2011smooth}
J.~J. Park and B.~Kuipers, ``A smooth control law for graceful motion of
  differential wheeled mobile robots in 2d environment,'' in \emph{IEEE Int.
  Conf. Robot. Autom.}, 2011, pp. 4896--4902.

\bibitem{ravankar2018path}
A.~Ravankar, A.~A. Ravankar, Y.~Kobayashi, Y.~Hoshino, and C.-C. Peng, ``Path
  smoothing techniques in robot navigation: State-of-the-art, current and
  future challenges,'' \emph{Sensors}, vol.~18, no.~9, p. 3170, 2018.

\bibitem{seki2005velocity}
H.~Seki and S.~Tadakuma, ``Velocity pattern generation for power assisted
  wheelchair based on jerk and acceleration limitation,'' in \emph{31st Annual
  Conference of IEEE Industrial Electronics Society}, 2005, p.~6.

\bibitem{nguyen2018path}
V.~T. Nguyen, C.~Sentouh, P.~Pudlo, and J.-C. Popieul, ``Path following
  controller for electric power wheelchair using model predictive control and
  transverse feedback linearization,'' in \emph{IEEE Int. Conf. Syst. Man.
  Cybern. (SMC)}, 2018, pp. 4319--4325.

\bibitem{TafrishiRObio2021}
S.~A. Tafrishi, R.~Ankit~A., J.~Salazar, and Y.~Hirata, ``A geometric assistive
  controller for the users of wheeled mobile robots without desired states,''
  in \emph{IEEE International Conference on Robotics and Biomimetics
  (ROBIO)}.\hskip 1em plus 0.5em minus 0.4em\relax IEEE, 2013, pp. 456--461.

\bibitem{CuiDarboux2010}
L.~Cui and J.~S. Dai, ``A darboux-frame-based formulation of spin-rolling
  motion of rigid objects with point contact,'' \emph{IEEE Trans. Robot.},
  vol.~26, no.~2, pp. 383--388, April 2010.

\bibitem{TafrishiMAMT2021}
S.~A. Tafrishi, M.~Svinin, and M.~Yamamoto, ``Darboux-frame-based
  parametrization for a spin-rolling sphere on a plane: A nonlinear
  transformation of underactuated system to fully-actuated model,'' \emph{Mech.
  Mach. Theory}, vol. 164, p. 104415, 2021.

\bibitem{DIfgeometry1976}
M.~P. do~Carmo, \emph{Differential Geometry of Curves and Surfaces},
  2nd~ed.\hskip 1em plus 0.5em minus 0.4em\relax Prentice-Hall, 1976.

\bibitem{Riemannian2002}
E.~Cartan, \emph{Riemannian Geometry in an Orthogonal Frame}, 1st~ed.\hskip 1em
  plus 0.5em minus 0.4em\relax World Scientific Pub Co Inc, 2002.

\bibitem{hart1988development}
S.~G. Hart and L.~E. Staveland, ``{Development of NASA-TLX (Task Load Index):
  Results of empirical and theoretical research},'' in \emph{Advances in
  psychology}.\hskip 1em plus 0.5em minus 0.4em\relax Elsevier, 1988, vol.~52,
  pp. 139--183.

\bibitem{SAtafrishiYoutubeICRA2022}
\BIBentryALTinterwordspacing
RoboHolic, ``[icra 2022] an assistive controller based on differential geometry
  for mobile robots (experiment).'' [Online]. Available:
  \url{https://www.youtube.com/watch?v=3iTsNa2hZ9s}
\BIBentrySTDinterwordspacing

\end{thebibliography}

%%%%%%%%%%%%%%%%%%%%%%%%%%%%%%%%%%%%%%%%%%%%%%%%%%%%%%%%%%%%%%%%%%%%%%%%%%%%%%%%

%%%%%%%%%%%%%%%%%%%%%%%%%%%%%%%%%%%%%%%%%%%%%%%%%%%%%%%%%%%%%%%%%%%%%%%%%%%%%%%%
%\section*{APPENDIX}

%Appendixes should appear before the acknowledgment.

%\section*{ACKNOWLEDGMENT}

%%%%%%%%%%%%%%%%%%%%%%%%%%%%%%%%%%%%%%%%%%%%%%%%%%%%%%%%%%%%%%%%%%%%%%%%%%%%%%%%

\end{document}